\documentclass[10pt,journal,compsoc]{IEEEtran}

\usepackage{bm}
\usepackage{verbatim}
\usepackage{amsmath,amsfonts}
\usepackage{cases}
\usepackage{booktabs} 
\usepackage{graphicx}
\usepackage{multirow} 
\usepackage{subfigure}
\usepackage{array}
\usepackage{textcomp}
\usepackage{stfloats}
\usepackage{url}
\usepackage{cite}
\usepackage{amsthm,amsmath,amssymb}
\usepackage{mathrsfs}
\usepackage{bbding}
\usepackage{xcolor}

\usepackage{caption}

\usepackage{booktabs}
\usepackage{threeparttable}

\newtheorem{theorem}{Theorem}
\newtheorem{corollary}{Corollary}
\newtheorem{definition}{Definition}

\usepackage{algorithm}
\usepackage{algorithmic}
\usepackage{hyperref}
\hypersetup{hypertex=true,
	colorlinks=true,
	linkcolor=blue, 
	anchorcolor=blue, 
	citecolor=blue}

\begin{document}

\title{One-step Bipartite Graph Cut: A Normalized Formulation and Its Application to Scalable Subspace Clustering}

\author{Si-Guo Fang, 
	Dong Huang,
	Chang-Dong Wang,
	and~Jian-Huang Lai,
	\IEEEcompsocitemizethanks{\IEEEcompsocthanksitem S.-G. Fang and D. Huang are with the College of Mathematics and Informatics, South China Agricultural University, Guangzhou, China. \protect\\
		E-mail: siguofang@hotmail.com, huangdonghere@gmail.com. \protect\\
		(Corresponding author: Dong Huang)
		\IEEEcompsocthanksitem C.-D. Wang and J.-H. Lai are with the School of Computer Science and Engineering,
		Sun Yat-sen University, Guangzhou, China, and also with Guangdong Key Laboratory of Information Security Technology, Guangzhou, China, and also with Key Laboratory of Machine Intelligence and Advanced Computing, Ministry of Education, China.\protect\\
		E-mail: changdongwang@hotmail.com, stsljh@mail.sysu.edu.cn.}
}

\markboth{}%
{Shell \MakeLowercase{\textit{et al.}}: Bare Demo of IEEEtran.cls for Computer Society Journals}

\IEEEtitleabstractindextext{%
\begin{abstract}
The bipartite graph structure has shown its promising ability in facilitating the subspace clustering and spectral clustering algorithms for large-scale datasets. To avoid the post-processing via $k$-means during the bipartite graph partitioning, the constrained Laplacian rank (CLR) is often utilized for constraining the number of connected components (i.e., clusters) in the bipartite graph, which, however, neglects the distribution (or normalization) of these connected components and may lead to imbalanced or even ill clusters. Despite the significant success of normalized cut (Ncut) in general graphs, it remains surprisingly an open problem how to enforce a one-step normalized cut for bipartite graphs, especially with linear-time complexity. In this paper, we first characterize a novel one-step bipartite graph cut (OBCut) criterion with normalized constraints, and theoretically prove its equivalence to a trace maximization problem. Then we extend this cut criterion to a scalable subspace clustering approach, where adaptive anchor learning, bipartite graph learning, and one-step normalized bipartite graph partitioning are simultaneously modeled in a unified objective function, and an alternating optimization algorithm is further designed to solve it in linear time. Experiments on a variety of general and large-scale datasets demonstrate the effectiveness and scalability of our approach.
\end{abstract}
	
}

\maketitle

\IEEEdisplaynontitleabstractindextext

\IEEEpeerreviewmaketitle

\section{Introduction}\label{sec:introduction}

\IEEEPARstart{D}{ata} clustering is one of the most fundamental topics in knowledge discovery and data mining, which aims to partition a set of data samples into a number of disjoint subsets, each referred to as a cluster. Among the clustering techniques that have been developed, the subspace clustering technique has been gaining increasing attention in recent years \cite{vidal2011subspace,tang19_tmm,SGL,FPMVS-CAG}, due to its ability to explore the topological relationship between data samples while tackling the so-called ``curse of dimensionality'' for high-dimensional data. 

The goal of subspace clustering is to pursue a self-representation matrix based on the self-expressive property \cite{fettal2023scalable,XU2023109152}.
Specifically, there have been several classical subspace clustering methods in the literature, including the sparse subspace clustering (SSC) \cite{SSC}, the low-rank representation (LRR) \cite{LRR}, and the least squares regression (LSR) \cite{LSR}. These subspace clustering methods \cite{SSC,LRR,LSR} typically perform two separate steps, i.e., (i) the similarity graph learning (via subspace learning) and (ii) the spectral partitioning, to obtain the clustering result, yet lack the ability to adaptively and jointly achieve the graph learning and graph partitioning. To bridge this gap, Li et al. \cite{LiSubspace} proposed a subspace clustering method which is able to jointly learn the similarity graph and the segmentation. In spite of the efforts to (partially) address the unified formulation problem, another common limitation to most of previous subspace clustering methods \cite{SSC,LRR,LSR,LiSubspace,You2016CVPR} is that they typically suffer from the cubic computational complexity, which significantly restricts their application in large-scale datasets.

Note that the \textit{subspace clustering} is generally associated with the \textit{spectral clustering}, where the spectral clustering is often adopted to partition the learned similarity matrix (via subspace learning) for the final clustering. To make the subspace/spectral clustering  feasible for large-scale datasets, the bipartite graph formulation has recently emerged as a promising strategy to greatly reduce the computational complexity of subspace/spectral clustering \cite{Nystrom,LSC,U-SPEC}. Typically, it first generates $M$ anchors (also known as representatives or landmarks) to represent the entire dataset with $k \le M\ll N$, where $k$ is the desired number of clusters. Then it constructs (or learns) a bipartite graph that connects the $N$ original samples and the $M$ anchors, which can be regarded as encoding the full sample-wise relationships through a small number of anchors \cite{FSSF} and is able to significantly alleviate the time and space complexity of subspace clustering and spectral clustering.
Specifically, the Nystr\"om approximation method \cite{Nystrom} randomly selects a certain number of anchors to construct the bipartite graph. Cai et al. \cite{LSC} proposed the landmark-based spectral clustering (LSC) method, which selects the anchors by performing the $k$-means clustering and then constructs a sparse affinity matrix (corresponding to a bipartite graph) for later spectral partitioning. Huang et al. \cite{U-SPEC} presented the ultra-scalable spectral clustering (U-SPEC) method, where a hybrid anchor selection strategy and a fast $K$-nearest neighbor approximation technique are devised to efficiently construct the bipartite graph and then the transfer cut \cite{TCut} is utilized to partition the bipartite graph.

Although the bipartite graph based subspace/spectral clustering methods \cite{Nystrom,LSC,U-SPEC} have achieved significant progress in reducing the computational complexity, yet most of them rely on some heuristic combinations of multiple separate steps, and are especially faced with two critical issues.
\begin{itemize}
	\item In the bipartite graph construction process, the bipartite graphs in previous works \cite{Nystrom,LSC,U-SPEC} are mostly predefined, which are separated from the later partitioning process and lack the desired ability of adaptive graph learning.
	\item In the bipartite graph partitioning process, they mostly require an additional $k$-means step to construct the clustering from the spectral embedding, which fail to directly learn the  discrete clustering structure and may be influenced by the instability of the $k$-means discretization.
\end{itemize}

Recently some efforts have been made to deal with the direct (or one-step) graph partitioning problem. 
A popular technique to directly learn the discrete clustering structure is the constrained Laplacian rank (CLR) strategy \cite{SFMC,CLR-single1,SGL}. It obtains the final clustering labels from the graph connectivity perspective by constraining the rank of the Laplacian matrix. More specifically, it typically learns a graph with a certain number of connected components, where each connected component naturally forms a final cluster.
For example, Nie et al. \cite{CLR-single1} proposed a graph-based clustering method based on the CLR strategy, which directly learns a similarity graph with a certain number of connected components.
Li et al. \cite{8341858} employed a rank-constrained similarity graph to recover the block-diagonal structure of an initial graph, where the learned embedding and the low-dimensional projection are jointly optimized.
Zhong et al. \cite{ZHONG2020127} imposed a rank constraint on the self-representation matrix,  and took into consideration both the global and local structures in subspace learning and graph regularization.
Note that the above-mentioned CLR-based methods \cite{CLR-single1,8341858,ZHONG2020127} are designed for the general graph (typically with a $N\times N$ similarity matrix), which may not be feasible for very large datasets. To alleviate the computational bottleneck,
Nie et al. \cite{co-clustering-nie2017} further proposed a co-clustering method, where the CLR constraint is imposed on the \textit{bipartite graph} and thus the computational complexity can be significantly reduced.
Kang et al. \cite{SGL} proposed a bipartite graph learning method with a connectivity constraint, where a structured bipartite graph can be adaptively learned in a subspace clustering framework.

These existing CLR-based clustering methods \cite{CLR-single1,8341858,ZHONG2020127,co-clustering-nie2017,SGL} aim to build a similarity graph with a desired number of connected components (or clusters), which, however, neglect the distribution (or normalization) of these connected components. In the conventional spectral clustering algorithms, such as the normalized cut (Ncut) \cite{Ncut}, the normalization of clusters plays an important role in avoiding the generation of some heavily imbalanced or even ill clusters (e.g., a cluster with a few or a single data sample). \textit{Yet surprisingly, under the bipartite graph setting, it remains an open problem how to simultaneously enforce bipartite graph learning and normalized (or balanced) partitioning without requiring additional post-processing, while maintaining high efficiency for large-scale datasets.}

To address this, this paper presents a novel \textbf{o}ne-step \textbf{b}ipartite graph \textbf{cut} (OBCut) approach, which for the first time, to our knowledge, formulates and solves the one-step bipartite graph learning and partitioning problem with normalized constraints.
Particularly, we theoretically characterize a novel bipartite graph cut criterion, which can be  equivalently transformed into a matrix trace form and is capable of balancing both the node size and the edge volume of each cluster. Theoretical analysis reveals the connection between the proposed bipartite graph cut criterion and two classical cut criteria (i.e., the RatioCut and the Ncut). Then, we integrate the bipartite graph cut criterion into an anchor-based subspace clustering framework, which simultaneously enforces adaptive anchor learning, bipartite graph learning, and normalized bipartite graph partitioning in a unified objective function. Further, an efficient optimization algorithm is designed to directly learn the discrete cluster indicator matrix without additional post-processing, which notably has linear time complexity in sample size. Extensive experiments are conducted on eight real-world general-scale and large-scale datasets, whose data sizes range from 832 to 195,537 and dimensions range from 7 to 30,000. The experimental results demonstrate the superiority of our OBCut approach over the state-of-the-art subspace/spectral clustering approaches.

The main contributions of this paper are summarized as follows.

\begin{itemize}
  \item A new normalized bipartite graph cut criterion is theoretically characterized, which can be equivalently transformed into a trace maximization problem and is featured by its ability to achieve one-step graph cut with the node size and the edge volume of each cluster simultaneously balanced.
  \item A scalable subspace clustering approach is proposed based on the new bipartite graph cut criterion, which formulates the adaptive anchor learning, the bipartite graph learning, and the one-step normalized bipartite graph partitioning into a unified optimization framework.
  \item An alternating minimization algorithm is designed to solve this optimization problem in linear time. Experiments on a variety of datasets have confirmed the advantageous performance of our approach over the state-of-the-art.
\end{itemize}

The remainder of this paper is organized as follows. The related works on spectral clustering, subspace clustering, and one-step clustering are reviewed in Section \ref{sec:related_work}. The proposed OBCut approach is described in Section \ref{sec:method}. The optimization algorithm and its theoretical analysis are provided in Section \ref{sec:opt}. The experimental results are reported in Section \ref{sec:experiments}. Finally, this paper is concluded in Section \ref{sec:conclusion}.

\section{Related Work}\label{sec:related_work}
In this section, the related works on spectral clustering, subspace clustering, and one-step clustering will be reviewed in Sections \ref{sec:SC}, \ref{sec:SubspaceClustering}, and \ref{sec:OSC}, respectively.

\subsection{Spectral Clustering}\label{sec:SC}
Spectral clustering has shown its advantage in discovering clusters with nonlinearly separable shapes. But the conventional spectral clustering typically suffers from its cubic time complexity, which restricts its applications in large-scale datasets. 

In recent years, some fast approximation methods (e.g., the bipartite graph based methods) have gained increasing popularity for alleviating the huge computational burden of spectral clustering \cite{ESCG,LSC,FastESC,EulerSC,U-SPEC,RKSC,DCDP-ASC}. 
For example,
Liu et al. \cite{ESCG} constructed a small set of supernodes from the original nodes in the similarity graph, and connected these supernodes with the original nodes to form a bipartite graph, based on which an efficient spectral clustering method is presented.
Cai et al. \cite{LSC} selected a set of landmarks (or anchors) via the $k$-means clustering, and proposed the landmark-based representation for large-scale spectral clustering.
He et al. \cite{FastESC} designed a fast large-scale spectral clustering method with the explicit feature mapping leveraged to speed up the eigenvector approximation.
Wu et al. \cite{EulerSC} utilized a positive Euler kernel to generate a non-negative similarity matrix and further developed an Euler spectral clustering method, which can be optimized by an efficient Stiefel-manifold-based gradient algorithm.
Huang et al. \cite{U-SPEC} proposed an ultra-scalable spectral clustering (U-SPEC) method based on hybrid anchor selection and fast $K$-nearest neighbor approximation.
Cheng et al. \cite{DCDP-ASC} designed an approximate spectral clustering method via dense cores and density peaks, which constructs a decision graph by computing the geodesic distances between the  dense cores and then expands the partitioning result of the dense cores to the data samples in the entire dataset.

\subsection{Subspace Clustering}\label{sec:SubspaceClustering}
Subspace clustering aims to learn a subspace representation matrix, upon which the similarity matrix can be derived and thus the final clustering can be obtained by partitioning this similarity matrix (typically via spectral clustering) \cite{Nie2020TKDE,SGL}. In subspace clustering, it is generally assumed that all data samples lie in multiple low-dimensional subspaces and can be expressed as a linear combination of the other samples in the same subspace \cite{Nie2020TKDE,SGL}.

Many subspace clustering works have been developed in the literature. For example, You et al. \cite{You2016CVPR} proposed a sparse subspace clustering method via the orthogonal matching pursuit algorithm.
Peng et al. \cite{PengxiSubspace} utilized a sparse L2-Graph to alleviate the potentially negative effects of the errors from the subspace representation.
Lu et al. \cite{Lu2018TPAMI} designed a block diagonal matrix induced regularizer to learn the self-representation for subspace clustering.
Chang et al. \cite{Chang2019ICASSP} employed the low-rank representation to learn a structured bipartite graph for subspace clustering, which can avoid the extra post-processing when obtaining the final clustering labels.
Nie et al. \cite{Nie2020TKDE} introduced a rank minimization problem, where a subspace indicator (i.e., the cluster indicator) can be learned by optimizing a relaxed piece-wise objective function.
Fan et al. \cite{Fan2021KDD} proposed a matrix factorization model for efficient subspace clustering, which can assign the data samples to the corresponding subspace directly.
Nie et al. \cite{LAPIN} integrated the anchor learning and the structured bipartite graph learning into a subspace clustering framework via CLR. Though some efforts have been made to directly learn the discrete clustering structure \cite{Chang2019ICASSP,LAPIN,SGL}, they typically seek to constrain the number of connected components in the graph (especially via CLR) yet lack the ability to consider the distribution (or normalization) of these connected components.

\subsection{One-step Clustering}\label{sec:OSC}
The subspace clustering is frequently associated with the spectral clustering, where the spectral clustering is performed on the learned similarity matrix (from the subspace representation) to obtain the final clustering. 

Conventional spectral clustering typically comply with the two-step formulation, where the spectral embedding is learned via eigen-decomposition in the first step and the $k$-means discretization is performed on the spectral embedding in the second step, which, however, cannot optimize these two steps simultaneously and may be negatively influenced by the instability of $k$-means \cite{yangDiscrete}.
Recently, some one-step spectral clustering methods have been proposed to obtain the discrete clustering solution without post-processing. 

One strategy is to employ the spectral rotation, which optimizes the \textit{continuous} spectral embedding and the \textit{discrete} cluster indicator matrix simultaneously to avoid the potential information loss that arises from the two-step methods \cite{SR7}. For example, Yang et al. \cite{SR2} imposed the nonnegative constraint on the spectral embedding, and presented a novel spectral clustering method with nonnegativity, discreteness, and discrimination. Pang et al. \cite{SR4} proposed a joint model to optimize the spectral embedding and the binary cluster indicator matrix simultaneously. Lu et al. \cite{SR7} integrated multiple kernel $k$-means (MKKM) and the spectral rotation into a unified framework. 
Another strategy is to directly compute the discrete cluster indicator matrix without relaxation. Chen et al. \cite{Direct1-DNC} directly solved the normalized cut and obtained the discrete cluster indicator matrix by optimizing the classical normalized cut model without extra post-processing. Some recent studies \cite{Direct3-uniform,UOMvSC} show that the classic  $k$-means clustering and the spectral clustering can be reformulated as a unified framework, where the clustering labels can be directly learned during their optimization process.
Despite the significant progress, these one-step spectral clustering methods are mostly designed for the general graph (typically with an $N\times N$ similarity matrix), which are not feasible for the bipartite graph. More recently, some attempts have been carried out to enable the one-step spectral clustering for the bipartite graph \cite{SGL}, which utilize the Laplacian low-rank constraint to control the number of connected components, but may still suffer from imbalanced or ill clusters due to their lack of the ability in conducting direct and normalized bipartite graph cut.

\section{Methodology}\label{sec:method}
In this section, we describe the proposed OBCut approach in detail. Specifically, the notations are summarized in Section \ref{Notations}. The adaptive bipartite graph construction via subspace learning is formulated in Section \ref{SubspaceLearning}. The normalized bipartite graph cut criterion is presented in Section \ref{NormalizedFormulation}. The computation of the new cut criterion is analyzed in Section \ref{sec:proof}. Finally, Section~\ref{Unified Formulation} provides the unified formulation of our OBCut approach.

\subsection{Notations}\label{Notations}
Throughout this paper, the set is written as uppercase blackboard bold, such as $\mathbb{R}$. The vector and the matrix are written as lowercase boldface and uppercase boldface, respectively. Given a matrix $\textbf{B}\in\mathbb{R}^{N\times M}$, its $(i,j)$-th entry is denoted as $b_{ij}$ or $\textbf{B}(i,j)$, and its $i$-th row and $j$-th column are denoted as $\textbf{b}_{i:}$ (or $\textbf{B}(i,:)$) and $\textbf{b}_{:j}$ (or $\textbf{B}(:,j)$), respectively. Let the transpose of matrix $\textbf{B}$ be denoted as $\textbf{B}^\top$, and the $F$-norm of $\textbf{B}$ be denoted as $||\textbf{B}||_F=\sqrt{\sum_{i=1}^N\sum_{j=1}^Mb_{ij}^2}=\sqrt{Tr(\textbf{B}^\top\textbf{B})}$, where $Tr(\cdot)$ is the trace of the matrix. Let $\textbf{I}$ denote the identity matrix, $\textbf{1}$ denote a column vector with all entries being one, and $\textbf{B}\ge 0$ denote that all the entries in this matrix are larger than or equal to zero.
For clarity, the main mathematical notations and their descriptions used in this paper are shown in Table \ref{tab:notations}.

\begin{table}[!t]
\begin{center}
\centering 
\caption{Summary of Notations} \vskip -0.08 in
\label{tab:notations}
\begin{tabular}{m{1.99cm}m{5.4cm}}
\toprule
Notations &Descriptions\\
\midrule
$\mathbb{S}_i$          &The sample set in the $i$-th cluster\\
$\mathbb{A}_j$          &The anchor set in the $j$-th cluster\\
$N$                                     &The number of samples\\
$M$                                     &The number of anchors\\
$k$                                     &The number of clusters\\
$d$                               &Dimension\\
$\textbf{X}\in\mathbb{R}^{d\times N}$       &Data matrix\\
$\textbf{A}\in\mathbb{R}^{d\times M}$             &The anchor matrix\\
$\textbf{B}\in\mathbb{R}^{N\times M}$       &The similarity matrix of a bipartite graph\\
$\textbf{Y}\in\{0,1\}^{N\times K}$      &The discrete cluster indicator matrix\\
$\textbf{H}\in\mathbb{R}^{M\times K}$   &The spectral embedding matrix for anchors\\
\bottomrule
\end{tabular}
\end{center}\vskip -0.15 in
\end{table}

\subsection{Adaptive Bipartite Graph Construction via Anchor-Based Subspace Learning}\label{SubspaceLearning}
Given a data set with $N$ data samples, let $\textbf{X}\in\mathbb{R}^{d\times N}$ denote its data matrix, where the $i$-th column is the feature vector of the $i$-th sample and $d$ is the dimension. 

In the conventional bipartite graph formulation, to select a set of anchors, the random sampling based selection and the $k$-means based selection are two of the most frequently-used strategies \cite{LSC,LMVSC}. However, the random sampling based anchor selection may not sufficiently reflect the overall distribution of the data \cite{TBGL-MVC}, while the $k$-means based selection is able to discover a set of more representative anchors but cannot well handle the non-linearly separable data.
Recently, some studies have gone beyond the conventional random sampling based or $k$-means based anchor selection to explore more effective strategies, such as the hybrid representative selection (HRS) \cite{U-SPEC}, the directly alternate sampling (DAS) \cite{SFMC}, and the variance-based de-correlation anchor selection (VDA) \cite{TBGL-MVC}. The HRS strategy performs the random sampling and the $k$-means clustering sequentially, which strikes a balance between the effectiveness of the $k$-means based selection and the efficiency of the random sampling based selection. Both the DAS and VDA strategies compute the score of each sample point by a self-defined function so as to measure the importance of the sample for anchor selection.

Despite the considerable progress, the previous works \cite{LSC,LMVSC,TBGL-MVC,U-SPEC,SFMC} mostly tend to select the anchors in a fixed manner, yet cannot go beyond the conventional (fixed) anchor selection to enforce the adaptive anchor learning or even the joint modeling of anchor learning and bipartite graph learning. 

In this paper,
we seek to jointly and adaptively learn the anchor matrix $\textbf{A}\in\mathbb{R}^{d\times M}$ and the bipartite graph $\textbf{B}\in\mathbb{R}^{N\times M}$, where $M$ is the number of anchors. It follows the basic assumption of subspace learning that each sample can be written as an affine or linear combination of the learned anchors, that is, $\textbf{X}=\textbf{A}\textbf{B}^\top+\textbf{E}$, where $\textbf{E}\in\mathbb{R}^{d\times N}$ is the error term. To adaptively learn the anchor matrix and the bipartite graph, we define the objective function as
\begin{align}
\label{eq:Subspace}
\min_{\textbf{A},\textbf{B}}~&||\textbf{X}-\textbf{A}\textbf{B}^\top||_F^2\notag\\
s.t.~&\textbf{B}\textbf{1}=\textbf{1},\textbf{B}\ge0,
\end{align}
where the constraints $\textbf{B}\textbf{1}=\textbf{1}$ and $\textbf{B}\ge0$ restrict each entry in $\textbf{B}$ to $0\le b_{ij}\le 1$. In the following, this objective function will serve as the graph learning term in our unified objective function in Section~\ref{Unified Formulation} .

\subsection{Normalized Formulation of Bipartite Graph Cut}\label{NormalizedFormulation}
Besides the bipartite graph \textit{learning}, we proceed to formulate the bipartite graph \textit{partitioning}, for which purpose we for the first time, to the best of our knowledge, theoretically characterize a one-step bipartite graph cut criterion with normalization and optimize it in linear time.

Previous bipartite graph based methods \cite{Nystrom,LSC,U-SPEC} generally involve a two-step process for graph partitioning, which first conduct eigen-decomposition on the bipartite graph (to obtain the spectral embedding by stacking the first $k$ eigen-vectors), and then build the final clustering labels by performing $k$-means on the spectral embedding. Instead of following the conventional two-step formulation, we propose a new one-step bipartite graph cut criterion in this section.
Specifically, we first propose a variant of $\emph{RatioCut}$ \cite{von2007tutorial} for the bipartite graph, and then, inspired by the normalized cut ($\emph{Ncut}$) \cite{Ncut}, extend this variant to a normalized formulation of the bipartite graph cut.

Given a bipartite graph $\textbf{G}(\mathbb{S},\mathbb{A},\textbf{B})$ with $\textbf{B}=(b_{ij})\in\mathbb{R}^{N\times M}$, where $\mathbb{S}$ is the sample set with $N$ elements, $\mathbb{A}$ is the anchor set with $M$ ($M\ll N$) elements, and $b_{ij}>0$ denotes the weight of the edge between the $i$-th sample and the $j$-th anchor. Let $b_{ij}=0$ if there is no edge between the $i$-th sample and the $j$-th anchor. Note that each edge in $\textbf{G}(\mathbb{S},\mathbb{A},\textbf{B})$ has one endpoint in $\mathbb{S}$ and one endpoint in $\mathbb{A}$. That is, there are no edges between two samples or between two anchors. For convenience, we can use the similarity matrix $\textbf{B}$ to represent the bipartite graph $\textbf{G}(\mathbb{S},\mathbb{A},\textbf{B})$.

In graph theory, for two disjoint sets $\mathbb{S}_i$ and $\mathbb{A}_j$, a $\emph{cut}$ between $\mathbb{S}_i$ and $\mathbb{A}_j$ is defined as
\begin{align}
	\label{eq:cut_orig}
cut(\mathbb{S}_i,\mathbb{A}_j)=\sum\limits_{n\in\mathbb{S}_i,m\in\mathbb{A}_j}\textbf{B}(n,m).
\end{align}
The sample set $\mathbb{S}$ and the anchor set $\mathbb{A}$ can be partitioned into multiple disjoint sets, i.e., $\mathbb{S}=\mathbb{S}_1\cup\cdots\cup\mathbb{S}_k$ with $\mathbb{S}_i\cap\mathbb{S}_j=\emptyset~(\forall i\neq j)$, and $\mathbb{A}=\mathbb{A}_1\cup\cdots\cup\mathbb{A}_k$ with $\mathbb{A}_i\cap\mathbb{A}_j=\emptyset~(\forall i\neq j)$, where $k$ is the number of clusters. Note that the elements of the pair $(\mathbb{S}_i,\mathbb{A}_i)$ belong to the same cluster. By removing edges between different clusters, the degree of similarity between different clusters can be formulated as the sum of the weights of the edges that have been removed. The simplest and most direct aim of clustering (or graph partitioning) is to minimize the similarity between all different clusters, that is
\begin{align}
\label{eq:cut}
&cut((\mathbb{S}_1,\mathbb{A}_1),\cdots,(\mathbb{S}_k,\mathbb{A}_k))\notag\\
=&\frac{1}{2}\sum\limits_{i=1}^k(cut(\mathbb{S}_i,\mathop{\cup}_{j\neq i}\mathbb{A}_j)+cut(\mathop{\cup}_{j\neq i}\mathbb{S}_j,\mathbb{A}_i)).
\end{align}
According to \cite{Ncut,dataclustering}, the minimum $\emph{cut}$ criterion tends to cut out imbalanced clusters, especially for the isolated nodes in the bipartite graph. Inspired by the $\emph{RatioCut}$, we formulate the following variant of the $\emph{cut}$ in \eqref{eq:cut}:
\begin{align}
\label{eq:Varcut1}
\sum\limits_{i=1}^k(\frac{cut(\mathbb{S}_i,\mathop{\cup}\limits_{j\neq i}\mathbb{A}_j)}{|\mathbb{S}_i|}+\frac{cut(\mathop{\cup}\limits_{j\neq i}\mathbb{S}_j,\mathbb{A}_i)}{|\mathbb{A}_i|}).
\end{align}
This definition in \eqref{eq:Varcut1} can balance the number of samples between different clusters, but it cannot balance the number of samples and that number of anchors in the same cluster. In view of this, we seek a partition of $\textbf{G}(\mathbb{S},\mathbb{A},\textbf{B})$, where the ratios of samples and anchors belonging to the same cluster should be as similar as possible. Thus we have the following variant of \eqref{eq:Varcut1}:
\begin{align}
\label{eq:Varcut2}
\sum\limits_{i=1}^k&(\frac{cut(\mathbb{S}_i,\mathop{\cup}\limits_{j\neq i}\mathbb{A}_j)}{|\mathbb{S}_i|}+\frac{cut(\mathop{\cup}\limits_{j\neq i}\mathbb{S}_j,\mathbb{A}_i)}{|\mathbb{A}_i|}\notag\\
&+cut(\mathbb{S}_i,\mathbb{A}_i)(\frac{1}{\sqrt{|\mathbb{S}_i|}}-\frac{1}{\sqrt{|\mathbb{A}_i|}})^2).
\end{align}
Let $\textbf{P}$ be the augmented graph of $\textbf{B}$ defined as \cite{SFMC}
 \begin{align}
 \label{eq:P}
\textbf{P}=\begin{bmatrix}
&\textbf{B}\\
\textbf{B}^\top&\\
\end{bmatrix}\in\mathbb{R}^{(N+M)\times(N+M)}.
\end{align}
The degree matrix $\textbf{D}=diag(\textbf{P}\textbf{1})$ is a diagonal matrix. It can also be written in a block-diagonal form, that is
\begin{align}
\label{eq:defD}
\textbf{D}=\begin{bmatrix}
\textbf{D}_{(N)}&\\
&\textbf{D}_{(M)}\\
\end{bmatrix},
\end{align}
where $\textbf{D}_{(N)}=diag(\textbf{B}\textbf{1})$ and $\textbf{D}_{(M)}=diag(\textbf{B}^\top\textbf{1})$. Thus the graph Laplacian of the bipartite graph $\textbf{G}(\mathbb{S},\mathbb{A},\textbf{B})$ can be written as $\textbf{L}=\textbf{D}-\textbf{P}$.

Note that the definition of Eq. \eqref{eq:Varcut2} does not consider the degree of each point (i.e., each node). Inspired by the $\emph{Ncut}$, it is intuitive that, not only the number of nodes in each cluster, but also the sum of edge weights (i.e., the sum of degrees of all nodes) in each cluster should be balanced. Thus we present the following bipartite graph cut criterion:
\begin{align}
\label{eq:BiCut}
Bip&artiteGraphCut((\mathbb{S}_1,\mathbb{A}_1),\cdots,(\mathbb{S}_k,\mathbb{A}_k))\notag\\
=\sum\limits_{i=1}^k&(\frac{cut(\mathbb{S}_i,\mathop{\cup}\limits_{j\neq i}\mathbb{A}_j)}{|\mathbb{S}_i|}+\frac{cut(\mathop{\cup}\limits_{j\neq i}\mathbb{S}_j,\mathbb{A}_i)}{|\mathbb{A}_i|}\notag\\
&+cut(\mathbb{S}_i,\mathbb{A}_i)(\frac{1}{\sqrt{|\mathbb{S}_i|}}-\frac{1}{\sqrt{|\mathbb{A}_i|}})^2\notag\\
&-\frac{\sum\limits_{n\in\mathbb{S}_i}\textbf{D}_{(N)}(n,n)}{|\mathbb{S}_i|}-\frac{\sum\limits_{m\in\mathbb{A}_i}\textbf{D}_{(M)}(m,m)}{|\mathbb{A}_i|}).
\end{align}

An important property of the proposed cut criterion in \eqref{eq:BiCut} is that it can be equivalently expressed as the matrix trace form, that is
\begin{align}
\label{eq:equivalentConclusion}
&\min~BipartiteGraphCut((\mathbb{S}_1,\mathbb{A}_1),\cdots,(\mathbb{S}_k,\mathbb{A}_k))\notag\\
\Leftrightarrow&\max\limits_{\bar{\textbf{Y}}_{(N)}\in Nind(\mathbb{S},k),\bar{\textbf{Y}}_{(M)}\in Nind(\mathbb{A},k)}Tr(\bar{\textbf{Y}}_{(N)}^\top\textbf{B}\bar{\textbf{Y}}_{(M)}),
\end{align}
where the definitions of $\bar{\textbf{Y}}_{(N)}$ and $\bar{\textbf{Y}}_{(M)}$, and the proof of the equivalence in \eqref{eq:equivalentConclusion} will be shown in Section \ref{sec:proof}.

\subsection{Computing of the Cut Criterion}\label{sec:proof}
In this section, we describe the definition of the normalized indicator and theoretically prove that the above-mentioned bipartite graph cut criterion can be equivalently expressed as a brief matrix trace form.

\begin{definition}
\label{def:Nind}
Given a partition of $\mathbb{S}$ into $k$ sets $\mathbb{S}_1,\cdots,\mathbb{S}_k$ with $|\mathbb{S}|=N$, we define the set of normalized indicators from the partition of $\mathbb{S}$ by Nind$(\mathbb{S},k)$. One of the normalized indicators is defined as
\begin{equation}
\label{eq:Nind}
\bar{\textbf{Y}}_{(N)}(i,j)=\left\{
\begin{aligned}
&\frac{1}{\sqrt{|\mathbb{S}_j|}}&\text{if}~i\in\mathbb{S}_j,\\
&0&\text{otherwise}.
\end{aligned}
\right.
\end{equation}
where $i=1,\cdots,N$ and $j=1,\cdots,k$. Then we have $Nind(\mathbb{S},k)=\{\bar{\textbf{Y}}_{(N)}|\bar{\textbf{Y}}_{(N)}^\top\bar{\textbf{Y}}_{(N)}=\textbf{I},\bar{\textbf{Y}}_{(N)}$ as defined in \eqref{eq:Nind}$\}$.
\end{definition}

Thereafter, we have the following theorem.

\begin{theorem}
\label{the:equivalent}
Given a bipartite graph $\textbf{G}(\mathbb{S},\mathbb{A},\textbf{B})$, with $|\mathbb{S}|=N$ and $|\mathbb{A}|=M$, the following conclusion holds:
\begin{align}
\label{eq:OBCut}
&\min~BipartiteGraphCut((\mathbb{S}_1,\mathbb{A}_1),\cdots,(\mathbb{S}_k,\mathbb{A}_k))\notag\\
\Leftrightarrow&\max\limits_{\bar{\textbf{Y}}_{(N)}\in Nind(\mathbb{S},k),\bar{\textbf{Y}}_{(M)}\in Nind(\mathbb{A},k)}Tr(\bar{\textbf{Y}}_{(N)}^\top\textbf{B}\bar{\textbf{Y}}_{(M)}).
\end{align}
\end{theorem}

\begin{proof}
Let $\bar{\textbf{Y}}=[\bar{\textbf{Y}}_{(N)};\bar{\textbf{Y}}_{(M)}]$, where $\bar{\textbf{Y}}_{(N)}\in Nind(\mathbb{S},k)$ and $\bar{\textbf{Y}}_{(M)}\in Nind(\mathbb{A},k)$. According to the definitions of $\textbf{D}$ in Eq. \eqref{eq:defD} and $\bar{\textbf{Y}}$, we have
\begin{align}
\label{eq:theorem_term2}
&\sum\limits_{i=1}^k(\frac{\sum\limits_{n\in\mathbb{S}_i}\textbf{D}_{(N)}(n,n)}{|\mathbb{S}_i|}+\frac{\sum\limits_{m\in\mathbb{A}_i}\textbf{D}_{(M)}(m,m)}{|\mathbb{A}_i|})\notag\\
=&Tr(\bar{\textbf{Y}}_{(N)}^\top\textbf{D}_{(N)}\bar{\textbf{Y}}_{(N)}+\bar{\textbf{Y}}_{(M)}^\top\textbf{D}_{(M)}\bar{\textbf{Y}}_{(M)})\notag\\
=&Tr(\bar{\textbf{Y}}^\top\textbf{D}\bar{\textbf{Y}}).
\end{align}
Moreover, for the cluster $(\mathbb{S}_i,\mathbb{A}_i)$, according to the definition of the $\emph{cut}$ in Eq. \eqref{eq:cut_orig}, we know that
\begin{align}
\label{eq:theorem_term1}
&\frac{cut(\mathbb{S}_i,\mathop{\cup}\limits_{j\neq i}\mathbb{A}_j)}{|\mathbb{S}_i|}+\frac{cut(\mathop{\cup}\limits_{j\neq i}\mathbb{S}_j,\mathbb{A}_i)}{|\mathbb{A}_i|}\notag\\
&+cut(\mathbb{S}_i,\mathbb{A}_i)(\frac{1}{\sqrt{|\mathbb{S}_i|}}-\frac{1}{\sqrt{|\mathbb{A}_i|}})^2\notag\\
=&\sum\limits_{n\in\mathbb{S}_i,m\notin\mathbb{A}_i}\frac{b_{nm}}{|\mathbb{S}_i|}+\sum\limits_{n\notin\mathbb{S}_i,m\in\mathbb{A}_i}\frac{b_{nm}}{|\mathbb{A}_i|}\notag\\
&+\sum\limits_{n\in\mathbb{S}_i,m\in\mathbb{A}_i}b_{nm}(\frac{1}{\sqrt{|\mathbb{S}_i|}}-\frac{1}{\sqrt{|\mathbb{A}_i|}})^2\notag\\
=&\sum\limits_{n\in\mathbb{S}_i,m\notin\mathbb{A}_i}b_{nm}(\frac{1}{\sqrt{|\mathbb{S}_i|}}-0)^2+\sum\limits_{n\notin\mathbb{S}_i,m\in\mathbb{A}_i}b_{nm}(0-\frac{1}{\sqrt{|\mathbb{A}_i|}})^2\notag\\
&+\sum\limits_{n\in\mathbb{S}_i,m\in\mathbb{A}_i}b_{nm}(\frac{1}{\sqrt{|\mathbb{S}_i|}}-\frac{1}{\sqrt{|\mathbb{A}_i|}})^2.
\end{align}
Further, according to the definitions of $\bar{\textbf{Y}}_{(N)}$ and $\bar{\textbf{Y}}_{(M)}$, we have that
\begin{align}
\label{eq:theorem_term1_var1}
\eqref{eq:theorem_term1}=&\sum\limits_{n=1}^N\sum\limits_{m=1}^Mb_{nm}(\bar{\textbf{Y}}_{(N)}(n,i)-\bar{\textbf{Y}}_{(M)}(m,i))^2\notag\\
=&\sum\limits_{n=1}^N\sum\limits_{m=1}^Mb_{nm}(\bar{\textbf{Y}}(n,i)-\bar{\textbf{Y}}(N+m,i))^2\notag\\
=&\frac{1}{2}(\sum\limits_{n=1}^N\sum\limits_{m=1}^Mb_{nm}(\bar{\textbf{Y}}(n,i)-\bar{\textbf{Y}}(N+m,i))^2\notag\\
&+\sum\limits_{m=1}^N\sum\limits_{n=1}^Mb_{mn}(\bar{\textbf{Y}}(N+n,i)-\bar{\textbf{Y}}(m,i))^2).
\end{align}
According to the definition of the augmented graph $\textbf{P}$ in Eq. \eqref{eq:P} and the important property \cite{von2007tutorial} of graph Laplacian $\textbf{L}=\textbf{D}-\textbf{P}$, we have that
\begin{align}
\label{eq:theorem_term1_var2}
\eqref{eq:theorem_term1_var1}=&\frac{1}{2}\sum\limits_{n=1}^{N+M}\sum\limits_{m=1}^{N+M}p_{nm}(\bar{y}_{ni}-\bar{y}_{mi})^2=\bar{\textbf{y}}_{:i}^\top\textbf{L}\bar{\textbf{y}}_{:i}.
\end{align}
From Eqs. \eqref{eq:theorem_term2} and \eqref{eq:theorem_term1_var2}, and the definition of the cut criterion in Eq. \eqref{eq:BiCut}, we have that
\begin{align}
&BipartiteGraphCut((\mathbb{S}_1,\mathbb{A}_1),\cdots,(\mathbb{S}_k,\mathbb{A}_k))\notag\\
=&\sum\limits_{i=1}^k\bar{\textbf{y}}_{:i}^\top\textbf{L}\bar{\textbf{y}}_{:i}-Tr(\bar{\textbf{Y}}^\top\textbf{D}\bar{\textbf{Y}})=Tr(\bar{\textbf{Y}}^\top\textbf{L}\bar{\textbf{Y}})-Tr(\bar{\textbf{Y}}^\top\textbf{D}\bar{\textbf{Y}})\notag\\
=&-Tr(\bar{\textbf{Y}}^\top\textbf{P}\bar{\textbf{Y}})=-2Tr(\bar{\textbf{Y}}_{(N)}^\top\textbf{B}\bar{\textbf{Y}}_{(M)})
\end{align}
Thereby, we prove the Theorem \ref{the:equivalent}.
\end{proof}

Furthermore, given the indicator matrices $\textbf{Y}\in\{0,1\}^{N\times k}$ and $\textbf{H}\in\{0,1\}^{M\times k}$, the Eq. \eqref{eq:OBCut} can be equivalently represented as follows:

\begin{equation}
\begin{aligned}
&\min~BipartiteGraphCut((\mathbb{S}_1,\mathbb{A}_1),\cdots,(\mathbb{S}_k,\mathbb{A}_k))\\
\Leftrightarrow&\left\{
\begin{aligned}
&\max_{\textbf{Y},\textbf{H}}~Tr((\textbf{Y}^\top\textbf{Y})^{-\frac{1}{2}}\textbf{Y}^\top\textbf{B}\textbf{H}(\textbf{H}^\top\textbf{H})^{-\frac{1}{2}})\\
&s.t.~\textbf{Y}\in\{0,1\}^{N\times k},\textbf{Y}\textbf{1}=\textbf{1};\textbf{H}\in\{0,1\}^{M\times k},\textbf{H}\textbf{1}=\textbf{1}.
\end{aligned}
\right.
\end{aligned}
\end{equation}

With $k\le M\ll N$, when $M=k$, we have $\textbf{H}^\top\textbf{H}=\textbf{I}$, and the following Corollary holds.

\begin{corollary}
Given a bipartite graph $\textbf{G}(\mathbb{S},\mathbb{A},\textbf{B})$, where $|\mathbb{S}|=N$ and $|\mathbb{A}|=M=k$, for the indicator $\textbf{Y}\in\{0,1\}^{N\times k}$ and $\textbf{H}\in\mathbb{R}^{M\times k}$, the following conclusion holds:
\begin{equation}
	\label{eq:cut_trace_form}
\begin{aligned}
&\min~BipartiteGraphCut((\mathbb{S}_1,\mathbb{A}_1),\cdots,(\mathbb{S}_k,\mathbb{A}_k))\\
\Leftrightarrow&\left\{
\begin{aligned}
&\max_{\textbf{Y},\textbf{H}}~Tr((\textbf{Y}^\top\textbf{Y})^{-\frac{1}{2}}\textbf{Y}^\top\textbf{B}\textbf{H})\\
&s.t.~\textbf{Y}\in\{0,1\}^{N\times k},\textbf{Y}\textbf{1}=\textbf{1};\textbf{H}^\top\textbf{H}=\textbf{I},\textbf{H}\ge0.
\end{aligned}
\right.
\end{aligned}
\end{equation}
\end{corollary}

\subsection{Unified Formulation of One-step Bipartite Graph Cut}\label{Unified Formulation}

In this section, we provide the unified formulation of our proposed OBCut approach, where the adaptive bipartite graph learning and the one-step normalized bipartite graph partitioning are simultaneously enforced.

Specifically, the bipartite graph $\textbf{B}$ can be learned via the objective \eqref{eq:Subspace}, and can be partitioned in a one-step manner via the proposed cut criterion in the objective~\eqref{eq:cut_trace_form}. By relaxing the objective~\eqref{eq:cut_trace_form}, we can rewrite it into a more concise optimization problem as
\begin{align}
\label{eq:relax-OBCut}
\max_{\textbf{Y},\textbf{H}}~&Tr((\textbf{Y}^\top\textbf{Y})^{-\frac{1}{2}}\textbf{Y}^\top\textbf{B}\textbf{H})\notag\\
s.t.~&\textbf{Y}\in\{0,1\}^{N\times k},\textbf{Y}\textbf{1}=\textbf{1};\textbf{H}^\top\textbf{H}=\textbf{I}.
\end{align}
Then we proceed to unify the adaptive anchor learning, the bipartite graph learning, and the one-step normalized bipartite graph partitioning in a joint learning framework. Formally, we have the overall objective function of OBCut as follows:
\begin{align}
\label{eq:obj-final}
&\min_{\textbf{A},\textbf{B},\textbf{Y},\textbf{H}}~||\textbf{X}-\textbf{A}\textbf{B}^\top||_F^2-\lambda Tr((\textbf{Y}^\top\textbf{Y})^{-\frac{1}{2}}\textbf{Y}^\top\textbf{B}\textbf{H})\notag\\
&s.t.~\textbf{B}\textbf{1}=\textbf{1},\textbf{B}\ge0;\textbf{Y}\in\{0,1\}^{N\times k},\textbf{Y}\textbf{1}=\textbf{1};\textbf{H}^\top\textbf{H}=\textbf{I}.
\end{align}
where $\lambda>0$ is a trade-off parameter between the anchor-based subspace learning and the bipartite graph cut.

\begin{theorem}
\label{the:upper}
Given a bipartite graph $\textbf{B}\in\mathbb{R}^{N\times M}$ with $|b_{ij}|\le\alpha$ ($\alpha>0$), let the sum of entries in $\textbf{y}_{:j}$ be denoted as $n_j$ ($n_j\ge1$). The upper bound of the objective function value in \eqref{eq:relax-OBCut} is $NM\alpha$.
\end{theorem}

\begin{proof}
Since $\textbf{H}^\top\textbf{H}=\textbf{I}$, we have $|h_{ij}|\le1$, $\forall i,j$. And we know that $(\textbf{Y}^\top\textbf{Y})^{-\frac{1}{2}}=diag([\frac{1}{\sqrt{n_1}},\frac{1}{\sqrt{n_2}},
\cdots,\frac{1}{\sqrt{n_k}}])$. Thus, we have
\begin{align}
&Tr((\textbf{Y}^\top\textbf{Y})^{-\frac{1}{2}}\textbf{Y}^\top\textbf{B}\textbf{H})=Tr(\begin{bmatrix}
\frac{\textbf{y}_{:1}^\top\textbf{B}}{\sqrt{n_1}}\\
\frac{\textbf{y}_{:2}^\top\textbf{B}}{\sqrt{n_2}}\\
\vdots\\
\frac{\textbf{y}_{:k}^\top\textbf{B}}{\sqrt{n_k}}\\
\end{bmatrix}\begin{bmatrix}
\textbf{h}_{:1},\textbf{h}_{:2},\cdots,\textbf{h}_{:k}\\
\end{bmatrix})\notag\\
=&\sum\limits_{j=1}^k\frac{\textbf{y}_{:j}^\top\textbf{B}\textbf{h}_{:j}}{\sqrt{n_j}}\le
\sum\limits_{j=1}^k\textbf{y}_{:j}^\top\textbf{B}\textbf{h}_{:j}\le\sum\limits_{j=1}^kn_jM\alpha=NM\alpha.
\end{align}

\end{proof}
According to Theorem \ref{the:upper}, we know that the lower bound of the objective function value in the minimization problem \eqref{eq:obj-final} is $-\lambda NM$. Remarkably, the objective \eqref{eq:relax-OBCut} can serve as an add-on module, and can well be integrated into other bipartite graph learning models so as to provide the capability of one-step normalized bipartite graph partitioning.

\section{Optimization and Theoretical Analysis}\label{sec:opt}
In this section, we design an alternating optimization algorithm to minimize the objective function \eqref{eq:obj-final} in Section \ref{Optimization of Problem}, and analyze its computational complexity in Sections~\ref{Computational Complexity Analysis}.

\subsection{Optimization of Problem}\label{Optimization of Problem}

\subsubsection{Update $\textbf{Y}$}
With the other variables fixed, the subproblem that only relates to $\textbf{Y}$ can be written as
\begin{align}
	\label{eq:upY}
	\max_{\textbf{Y}}~&Tr((\textbf{Y}^\top\textbf{Y})^{-\frac{1}{2}}\textbf{Y}^\top\textbf{B}\textbf{H})\notag\\
	s.t.~&\textbf{Y}\in\{0,1\}^{N\times K},\textbf{Y}\textbf{1}=\textbf{1}.
\end{align}
Since $\textbf{Y}\in\{0,1\}^{N\times K}$ and $\textbf{Y}\textbf{1}=\textbf{1}$, we know that $(\textbf{Y}^\top\textbf{Y})^{-\frac{1}{2}}$ is a diagonal matrix\footnote{When $\textbf{y}_{:j}=\textbf{0}$, the $j$-th diagonal entry of the diagonal matrix $\textbf{Y}^\top\textbf{Y}$ is equal to $0$. The equation $\textbf{y}_{:j}=\textbf{0}$ indicates no sample belongs to the $j$-th cluster, which, however, is not desired and would cause the division-by-zero error in calculating $(\textbf{Y}^\top\textbf{Y})^{-\frac{1}{2}}$. To avoid this situation, we can use $(\textbf{Y}^\top\textbf{Y}+\epsilon\textbf{I})^{-\frac{1}{2}}$ instead of $(\textbf{Y}^\top\textbf{Y})^{-\frac{1}{2}}$, where $\epsilon>0$ is a very small constant. It is obvious that $(\textbf{Y}^\top\textbf{Y}+\epsilon\textbf{I})^{-\frac{1}{2}}\rightarrow(\textbf{Y}^\top\textbf{Y})^{-\frac{1}{2}}$ when $\epsilon\rightarrow0$.}. Let $\textbf{Q}=\textbf{B}\textbf{H}$, the subproblem \eqref{eq:upY} can be transformed as
\begin{align}
	\label{eq:upY1}
	\max_{\textbf{Y}}~&Tr((\textbf{Y}(\textbf{Y}^\top\textbf{Y})^{-\frac{1}{2}})^\top\textbf{Q})=\sum\limits_{j=1}^K\frac{\sum_{i=1}^Ny_{ij}q_{ij}}{\sqrt{\textbf{y}_{:j}^\top\textbf{y}_{:j}}}\notag\\
	s.t.~&\textbf{Y}\in\{0,1\}^{N\times K},\textbf{Y}\textbf{1}=\textbf{1}.
\end{align}
According to \cite{UOMvSC}, since $\sqrt{\textbf{y}_{:j}^\top\textbf{y}_{:j}}$ involves every row of $\textbf{Y}$, we can
optimize $\textbf{Y}$ row by row.

For updating the $i$-th row vector $\textbf{y}_{i:}$ in $\textbf{Y}$, we can explore the increase of the objective function values in objective \eqref{eq:upY1}, when $\textbf{y}_{i:}$ changes from $\textbf{y}_{i:}=\textbf{0}$ to $y_{ij}=1$. Whether the initial value of $y_{ij}$ is $0$ or $1$, the increment can be unified into the following form as
\begin{align}
\Delta_{ij}=\frac{\sum_{s=1}^Ny_{sj}q_{sj}+q_{ij}(1-y_{ij})}{\sqrt{\textbf{y}_{:j}^\top\textbf{y}_{:j}+(1-y_{ij})}}-\frac{\sum_{s=1}^Ny_{sj}q_{sj}+q_{ij}y_{ij}}{\sqrt{\textbf{y}_{:j}^\top\textbf{y}_{:j}-y_{ij}}}.
\end{align}
Then $\textbf{y}_{i:}$ can be updated as
\begin{align}
\label{eq:UpdateY}
y_{ij}=<j=\mathop{\arg\max}\limits_{l\in\{1,2,\cdots,K\}}\Delta_{il}>,
\end{align}
where $<\cdot>$ returns $0$ when the argument is false or $1$ otherwise.

\subsubsection{Update $\textbf{H}$}
With the other variables fixed, the subproblem that only relates to $\textbf{H}$ can be written as
\begin{align}
    \label{eq:upH}
    \max_{\textbf{H}^\top\textbf{H}=\textbf{I}}~&Tr(\textbf{H}^\top\textbf{B}^\top\textbf{Y}(\textbf{Y}^\top\textbf{Y})^{-\frac{1}{2}}).
\end{align}
According to \cite{OPPnie}, we can have the optimal solution to the objective \eqref{eq:upH} as
\begin{align}
\label{eq:UpdateH}
\textbf{H}=\textbf{U}\textbf{V}^\top,
\end{align}
where $\textbf{U}\in\mathbb{R}^{M\times K}$ and $\textbf{V}\in\mathbb{R}^{K\times K}$ are obtained from the compact SVD \cite{GPInie} of $\textbf{B}^\top\textbf{Y}(\textbf{Y}^\top\textbf{Y})^{-\frac{1}{2}}=\textbf{U}\Sigma \textbf{V}^\top$.

\subsubsection{Update $\textbf{B}$}
With the other variables fixed, the subproblem that only relates to $\textbf{B}$ can be written as
\begin{align}
\label{eq:upP}
\min_{\textbf{B}}~&||\textbf{X}-\textbf{A}\textbf{B}^\top||_F^2-\lambda Tr((\textbf{Y}^\top\textbf{Y})^{-\frac{1}{2}}\textbf{Y}^\top\textbf{B}\textbf{H})\notag\\
s.t.~&\textbf{B}\textbf{1}=\textbf{1},\textbf{B}\ge0.
\end{align}
Since the optimization of each row of $\textbf{B}$ in subproblem \eqref{eq:upP} is independent, we can optimize $\textbf{B}$ in a row-by-row manner.

Let $\hat{\textbf{Q}}=\textbf{H}(\textbf{Y}^\top\textbf{Y})^{-\frac{1}{2}}\textbf{Y}^\top$, the optimization of the $i$-th row of $\textbf{B}$ (i.e. $\textbf{b}_{i:}$) can be formulated as follows:
\begin{align}
\label{eq:upP1}
\min_{\textbf{b}_{i:}}~&||\textbf{x}_{:i}-\textbf{A}\textbf{b}_{i:}^\top||_2^2-\lambda\textbf{b}_{i:}\hat{\textbf{q}}_{:i}\notag\\
s.t.~&\textbf{b}_{i:}\textbf{1}=1,\textbf{b}_{i:}\ge0.
\end{align}

Thereafter, the optimization problem \eqref{eq:upP1} can be rewritten as the following Quadratic Programming problem
\begin{align}
\label{eq:upP2}
\min_{\textbf{b}_{i:}}~&\textbf{b}_{i:}\hat{\textbf{H}}\textbf{b}_{i:}^\top-\textbf{b}_{i:}\textbf{f}\notag\\
s.t.~&\textbf{b}_{i:}\textbf{1}=1,\textbf{b}_{i:}\ge0,
\end{align}
where $\hat{\textbf{H}}=\textbf{A}^\top\textbf{A}$, and $\textbf{f}=2\textbf{A}^\top\textbf{x}_{:i}+\lambda\hat{\textbf{q}}_{:i}$. According to \cite{SFMC}, the problem \eqref{eq:upP2} can be optimized by the augmented Lagrangian multiplier (ALM) method.

\subsubsection{Update $\textbf{A}$}
With the other variables fixed, the subproblem that only relates to $\textbf{A}$ can be written as
\begin{align}
	\label{eq:upA}
	\min_{\textbf{A}}~&||\textbf{X}-\textbf{A}\textbf{B}^\top||_F^2.	
\end{align}

It is obvious that the subproblem \eqref{eq:upA} is a convex optimization problem \cite{LAPIN}. The optimal solution can be obtained by setting the derivative w.r.t. $\textbf{A}$ to zero, so we have\footnote{When $\textbf{B}^\top\textbf{B}$ is not invertible, we can use the Moore-Penrose inverse $(\textbf{B}^\top\textbf{B})^{+}$.}
\begin{align}
\label{eq:UpdateA}
	\textbf{A}=\textbf{X}\textbf{B}(\textbf{B}^\top\textbf{B})^{-1}.
\end{align}

For clarity, the overall process of our OBCut approach is described in Algorithm~\ref{algorithm}.

\begin{algorithm}
\renewcommand{\algorithmicrequire}{\textbf{Input:}}
\renewcommand{\algorithmicensure}{\textbf{Preparation:}}
\caption{One-step bipartite graph cut (OBCut)}\label{algorithm}
\begin{algorithmic}[1]
\REQUIRE The normalized data matrix $\textbf{X}$, the cluster numbers $k$, the number of anchors $M\ge k$, and parameter $\lambda>0$.
\renewcommand{\algorithmicensure}{\textbf{Initialization:}}
\ENSURE  Use $k$-means to initialize $\textbf{A}$ and $\textbf{Y}$. Initialize $\textbf{B}$ by solving problem  \eqref{eq:Subspace} and $K$-nearest-neighbor ($K$-NN) bipartite graph ($K=5$). Initialize $\textbf{H}$ by Eq. \eqref{eq:UpdateH}.
\REPEAT
    \STATE Update $\textbf{Y}$ by Eq. \eqref{eq:UpdateY}.
    \STATE Update $\textbf{H}$ by Eq. \eqref{eq:UpdateH}.
    \STATE Update $\textbf{B}$ by solving problem \eqref{eq:upP2}.
    \STATE Update $\textbf{A}$ by Eq. \eqref{eq:UpdateA}.
\UNTIL{Convergence or maximum iteration reached}
\renewcommand{\algorithmicensure}{\textbf{Output:}}
\ENSURE  The clustering results $\textbf{Y}$.
\end{algorithmic}
\end{algorithm}

\subsection{Computational Complexity Analysis}\label{Computational Complexity Analysis}
In this section, we analyze the computational complexity of our OBCut approach. First, the initialization of the anchor set by $k$-means takes $O(NMdt_1)$ time, where $t_1$ is the number of iterations. The initialization of the bipartite graph takes $O(NMK)$ time \cite{LSC}, where $K$ is the number of nearest neighbors. In each iteration, it takes $O(Nk)$ time to update $\textbf{Y}$. When updating $\textbf{H}$, it takes $O(MNk)$ time to calculate $\textbf{B}^\top\textbf{Y}(\textbf{Y}^\top\textbf{Y})^{-\frac{1}{2}}$, and $O(Mk^2)$ time to perform the SVD and update $\textbf{H}$. Thus, the total computational complexity of updating $\textbf{H}$ is $O(Mk^2+MNk)$. When solving the problem \eqref{eq:upP2}, it takes $O(NMk)$ time to calculate $\hat{\textbf{Q}}$. It costs $O(NM^2d)$ time to update $\textbf{B}$, so the total time complexity of updating $\textbf{B}$ is $O(NMk+NM^2d)$. For updating $\textbf{A}$, it costs $O(dNM+dM^2+NM^2+M^3)$ time in each iteration.

Therefore, the overall computational complexity of OBCut is $O(NMK+NMdt_1+t_2(NMk+NM^2d+Mk^2+M^3))$, where $t_2$ is the number of iterations. With $K$, $k$, $t_1$, $t_2$ being small constants and $M\ll N$, the computational complexity of OBCut is linear to the number of samples $N$.

\section{Experiments}\label{sec:experiments}
In this section, we conduct experiments to evaluate the the proposed OBCut approach against several state-of-the-art subspace/spectral clustering approaches on multiple real-world datasets. All experiments are conducted on a computer with an Intel i9-12900KF CPU and 16GB of RAM.

\subsection{Datasets and Evaluation Metrics}
In the experiments, we evaluate the proposed method and the baseline methods on eight real-world datasets, namely, Leeds \cite{Leeds}, MPEG-7 \cite{Mpeg7}, Yale, News Group-20 (NG-20) \cite{NG20}, Abalone \cite{Abalone}, Letter Recognition (LR) \cite{ECPCS}, Youtube Faces-50 (YTF-50) \cite{FastMICE}, and Youtube Faces-100 (YTF-100) \cite{FastMICE}, whose data sizes range from 832 to 195,537. The details of these benchmark datasets are given in Table \ref{tab:datasets}.

To quantitatively compare the clustering results by different methods, we adopt three well-known evaluation metrics, i.e., the normalized mutual information (NMI) \cite{FastMICE}, the accuracy (ACC) \cite{UDBGL}, and the purity (PUR) \cite{Liang2022}. For these three metrics, larger values indicate better clustering results. 

\begin{table}[!t]
	\centering
	\caption{Description of the benchmark datasets}\vskip -0.08 in
	\label{tab:datasets}
	\renewcommand\arraystretch{1.2}
	\begin{tabular}{p{1.7cm}<{\centering}|p{1.8cm}<{\centering}p{1.4cm}<{\centering}p{1.3cm}<{\centering}}
		\toprule
		Dataset &\#Sample    &Dimension  &\#Class \\
		\midrule
		Leeds	&832	&30,000	&10	\\
		MPEG-7	&1,400	&6,000	&70	\\
		Yale	&1,755	&1,200	&3	\\
		NG-20	&3,970	&8,014	&4	\\
		Abalone	&4,177	&7	&28	\\
		LR	&20,000	&16	&26	\\
		YTF-50	&126,054	&512	&50	\\
		YTF-100	&195,537	&512	&100	\\
		\bottomrule
	\end{tabular}\vskip -0.1in
\end{table}

\subsection{Baseline Methods and Experimental Settings}

In our experiments, we compare OBCut against nine spectral clustering and subspace clustering methods, which are listed below.
\begin{itemize}
	\item \textbf{SSC} \cite{SSC}: Sparse subspace clustering.
	\item \textbf{ESCG} \cite{ESCG}: Efficient spectral clustering on graphs.
	\item \textbf{SSC-OMP} \cite{You2016CVPR}: Sparse subspace clustering by orthogonal matching pursuit.
	\item \textbf{LSC} \cite{LSC}: Landmark-based spectral clustering.
	\item \textbf{FastESC} \cite{FastESC}: Fast explicit spectral clustering.
	\item \textbf{EulerSC} \cite{EulerSC}: Euler spectral clustering.
	\item \textbf{U-SPEC} \cite{U-SPEC}: Ultra-scalable spectral clustering.
	\item \textbf{RKSC} \cite{RKSC}: Refined $K$-nearest neighbor graph for spectral clustering.
	\item \textbf{DCDP-ASC} \cite{DCDP-ASC}: Approximate spectral clustering based on dense
	cores and density peaks.
\end{itemize}

Given a dataset, we perform each test method 20 times and report its average scores (w.r.t. NMI, ACC, and PUR). For our OBCut method, the number of anchors $M=100$ is used for all the datasets, and its trade-off parameter is tuned in the range of $[10^{-5},10^{-4},\cdots,10^{5}]$. Similarly, the hyper-parameters in the baseline methods are also tuned in the range of $[10^{-5},10^{-4},\cdots,10^{5}]$, unless some specific tuning range is given in their corresponding papers.

\begin{table*}[!th]\vskip -0.08 in
	\centering
	\newcommand{\tabincell}[2]{\begin{tabular}{@{}#1@{}}#2\end{tabular}}
	\caption{Average NMI(\%) over 20 runs by different clustering methods. The best two scores in each row are highlighted in \textbf{bold}, while the best one in [\textbf{bold and brackets}].}\vskip -0.08in
	\label{tab:NMI}
		\renewcommand\arraystretch{1.01}
		\resizebox{\textwidth}{!}{
			\begin{tabular}{m{1.1cm}|m{1.3cm}<{\centering}m{1.3cm}<{\centering}m{1.3cm}<{\centering}m{1.3cm}<{\centering}m{1.3cm}<{\centering}m{1.3cm}<{\centering}m{1.3cm}<{\centering}m{1.4cm}<{\centering}m{1.45cm}<{\centering}|m{1.46cm}<{\centering}}
				\toprule
				Datasets	&SSC	&ESCG	&SSC-OMP	&LSC	&FastESC	&EulerSC	&U-SPEC	&RKSC	&DCDP-ASC	&\textbf{OBCut}	\\
				\midrule
				Leeds	&N/A	&\textbf{13.05}$_{\pm1.01}$	&7.45$_{\pm0.23}$	&11.70$_{\pm1.53}$	&9.99$_{\pm1.48}$	&6.71$_{\pm0.71}$	&12.71$_{\pm0.60}$	&12.59$_{\pm0.88}$	&1.28$_{\pm0.00}$	&[\textbf{27.71}$_{\pm0.00}$]	\\
				MPEG-7	&70.39$_{\pm0.81}$	&61.11$_{\pm1.52}$	&[\textbf{73.30}$_{\pm0.60}$]	&60.67$_{\pm1.11}$	&47.64$_{\pm1.45}$	&60.56$_{\pm0.70}$	&63.85$_{\pm1.13}$	&58.38$_{\pm2.07}$	&6.85$_{\pm0.00}$	&\textbf{71.96}$_{\pm0.00}$	\\
				Yale	&[\textbf{100.00}$_{\pm0.00}$]	&[\textbf{100.00}$_{\pm0.00}$]	&[\textbf{100.00}$_{\pm0.00}$]	&80.25$_{\pm15.42}$	&97.94$_{\pm4.08}$	&78.90$_{\pm0.00}$	&64.25$_{\pm15.41}$	&72.89$_{\pm0.00}$	&[\textbf{100.00}$_{\pm0.00}$]	&[\textbf{100.00}$_{\pm0.00}$]	\\
				NG-20	&0.42$_{\pm0.00}$	&0.31$_{\pm0.11}$	&\textbf{9.96}$_{\pm5.43}$	&1.47$_{\pm1.13}$	&0.09$_{\pm0.02}$	&0.08$_{\pm0.04}$	&0.16$_{\pm0.17}$	&0.63$_{\pm0.00}$	&0.45$_{\pm0.00}$	&[\textbf{18.83}$_{\pm0.00}$]	\\
				Abalone	&13.09$_{\pm0.11}$	&16.59$_{\pm0.42}$	&6.86$_{\pm0.22}$	&15.13$_{\pm0.25}$	&16.64$_{\pm2.21}$	&7.26$_{\pm0.32}$	&15.38$_{\pm0.26}$	&0.65$_{\pm0.00}$	&\textbf{16.80}$_{\pm0.00}$	&[\textbf{18.07}$_{\pm0.00}$]	\\
				LR	&N/A	&31.24$_{\pm1.30}$	&2.51$_{\pm0.13}$	&33.06$_{\pm1.31}$	&34.72$_{\pm5.93}$	&24.71$_{\pm0.86}$	&\textbf{34.86}$_{\pm0.71}$	&7.92$_{\pm0.05}$	&22.39$_{\pm0.00}$	&[\textbf{37.43}$_{\pm0.00}$]	\\
				YTF-50	&N/A	&N/A	&4.28$_{\pm0.13}$	&67.13$_{\pm1.76}$	&71.32$_{\pm6.38}$	&17.25$_{\pm1.85}$	&\textbf{74.92}$_{\pm0.98}$	&69.28$_{\pm0.00}$	&51.44$_{\pm0.00}$	&[\textbf{82.60}$_{\pm0.00}$]	\\
				YTF-100	&N/A	&N/A	&3.89$_{\pm0.22}$	&58.86$_{\pm0.90}$	&\textbf{68.17}$_{\pm0.74}$	&66.59$_{\pm0.44}$	&66.13$_{\pm0.75}$	&53.06$_{\pm0.00}$	&7.77$_{\pm0.00}$	&[\textbf{81.53}$_{\pm0.00}$]	\\
				\midrule
				Avg.score	&-	&-	&26.03	&41.03	&\textbf{43.31}	&32.76	&41.53	&34.42	&25.87	&[\textbf{54.77}]	\\
				Avg.rank	&6.88	&5.25	&5.63	&5.13	&5.13	&7.13	&\textbf{4.75}	&6.63	&5.88	&[\textbf{1.13}]	\\
				\bottomrule
			\end{tabular}
		}
	\begin{tablenotes}
		\footnotesize
		\item[1] Note that N/A indicates the out-of-memory error.
	\end{tablenotes}
\end{table*}

\begin{table*}[!th]
	\centering
	\newcommand{\tabincell}[2]{\begin{tabular}{@{}#1@{}}#2\end{tabular}}
	\caption{Average ACC(\%) over 20 runs by different clustering methods. The best two scores in each row are highlighted in \textbf{bold}, while the best one in [\textbf{bold and brackets}].}\vskip -0.08in
	\label{tab:ACC}
		\renewcommand\arraystretch{1.01}
		\resizebox{\textwidth}{!}{
			\begin{tabular}{m{1.1cm}|m{1.3cm}<{\centering}m{1.3cm}<{\centering}m{1.3cm}<{\centering}m{1.3cm}<{\centering}m{1.3cm}<{\centering}m{1.3cm}<{\centering}m{1.3cm}<{\centering}m{1.4cm}<{\centering}m{1.45cm}<{\centering}|m{1.46cm}<{\centering}}
				\toprule
				Datasets	&SSC	&ESCG	&SSC-OMP	&LSC	&FastESC	&EulerSC	&U-SPEC	&RKSC	&DCDP-ASC	&\textbf{OBCut}	\\
				\midrule
				Leeds	&N/A	&23.73$_{\pm0.85}$	&17.92$_{\pm0.43}$	&22.82$_{\pm1.03}$	&21.53$_{\pm1.76}$	&19.32$_{\pm0.96}$	&22.99$_{\pm1.05}$	&\textbf{23.97}$_{\pm0.95}$	&12.74$_{\pm0.00}$	&[\textbf{27.16}$_{\pm0.00}$]	\\
				MPEG-7	&50.77$_{\pm1.16}$	&42.62$_{\pm1.41}$	&[\textbf{55.00}$_{\pm1.07}$]	&39.04$_{\pm1.69}$	&32.29$_{\pm2.08}$	&42.55$_{\pm1.12}$	&44.50$_{\pm1.17}$	&40.40$_{\pm2.64}$	&4.43$_{\pm0.00}$	&\textbf{54.36}$_{\pm0.00}$	\\
				Yale	&[\textbf{100.00}$_{\pm0.00}$]	&[\textbf{100.00}$_{\pm0.00}$]	&[\textbf{100.00}$_{\pm0.00}$]	&86.59$_{\pm14.15}$	&99.47$_{\pm1.36}$	&92.93$_{\pm0.00}$	&59.07$_{\pm17.64}$	&85.13$_{\pm0.00}$	&[\textbf{100.00}$_{\pm0.00}$]	&[\textbf{100.00}$_{\pm0.00}$]	\\
				NG-20	&25.29$_{\pm0.00}$	&27.52$_{\pm0.34}$	&\textbf{33.41}$_{\pm4.15}$	&28.02$_{\pm2.16}$	&25.13$_{\pm0.03}$	&26.31$_{\pm0.43}$	&25.25$_{\pm0.25}$	&25.57$_{\pm0.00}$	&28.44$_{\pm0.00}$	&[\textbf{40.23}$_{\pm0.00}$]	\\
				Abalone	&12.26$_{\pm0.15}$	&15.91$_{\pm1.14}$	&13.72$_{\pm0.47}$	&13.18$_{\pm0.42}$	&\textbf{21.42}$_{\pm3.09}$	&13.13$_{\pm0.53}$	&13.79$_{\pm0.65}$	&16.57$_{\pm0.00}$	&19.56$_{\pm0.00}$	&[\textbf{23.22}$_{\pm0.00}$]	\\
				LR	&N/A	&24.24$_{\pm1.13}$	&5.80$_{\pm0.09}$	&25.86$_{\pm1.83}$	&\textbf{26.66}$_{\pm4.13}$	&22.65$_{\pm1.03}$	&[\textbf{26.86}$_{\pm1.29}$]	&9.68$_{\pm0.01}$	&18.34$_{\pm0.00}$	&26.51$_{\pm0.00}$	\\
				YTF-50	&N/A	&N/A	&6.14$_{\pm0.08}$	&54.39$_{\pm3.31}$	&59.09$_{\pm5.88}$	&22.85$_{\pm1.74}$	&\textbf{65.62}$_{\pm2.00}$	&62.85$_{\pm0.00}$	&41.04$_{\pm0.00}$	&[\textbf{68.07}$_{\pm0.00}$]	\\
				YTF-100	&N/A	&N/A	&4.30$_{\pm0.04}$	&39.53$_{\pm1.58}$	&53.63$_{\pm1.65}$	&\textbf{56.33}$_{\pm1.23}$	&49.49$_{\pm1.39}$	&45.43$_{\pm0.00}$	&8.55$_{\pm0.00}$	&[\textbf{67.74}$_{\pm0.00}$]	\\
				\midrule
				Avg.score	&-	&-	&29.54	&38.68	&\textbf{42.40}	&37.01	&38.45	&38.70	&29.14	&[\textbf{50.91}]	\\
				Avg.rank	&7.50	&5.25	&5.50	&6.00	&5.25	&6.25	&\textbf{5.00}	&5.63	&5.75	&[\textbf{1.38}]	\\
				\bottomrule
			\end{tabular}
		}
	\begin{tablenotes}
		\footnotesize
		\item[1] Note that N/A indicates the out-of-memory error.
	\end{tablenotes}
\end{table*}

\begin{table*}[!th]
	\centering
	\newcommand{\tabincell}[2]{\begin{tabular}{@{}#1@{}}#2\end{tabular}}
	\caption{Average PUR(\%) over 20 runs by different clustering methods. The best two scores in each row are highlighted in \textbf{bold}, while the best one in [\textbf{bold and brackets}].}\vskip -0.08in
	\label{tab:PUR}
		\renewcommand\arraystretch{1.01}
		\resizebox{\textwidth}{!}{
			\begin{tabular}{m{1.1cm}|m{1.3cm}<{\centering}m{1.3cm}<{\centering}m{1.3cm}<{\centering}m{1.3cm}<{\centering}m{1.3cm}<{\centering}m{1.3cm}<{\centering}m{1.3cm}<{\centering}m{1.4cm}<{\centering}m{1.45cm}<{\centering}|m{1.46cm}<{\centering}}
				\toprule
				Datasets	&SSC	&ESCG	&SSC-OMP	&LSC	&FastESC	&EulerSC	&U-SPEC	&RKSC	&DCDP-ASC	&\textbf{OBCut}	\\
				\midrule
				Leeds	&N/A	&\textbf{27.30}$_{\pm0.98}$	&20.84$_{\pm0.46}$	&26.02$_{\pm1.32}$	&24.07$_{\pm2.07}$	&20.90$_{\pm0.87}$	&26.47$_{\pm0.78}$	&26.95$_{\pm0.83}$	&13.10$_{\pm0.00}$	&[\textbf{29.09}$_{\pm0.00}$]	\\
				MPEG-7	&57.01$_{\pm0.90}$	&46.47$_{\pm1.42}$	&\textbf{58.68}$_{\pm0.97}$	&41.93$_{\pm1.50}$	&34.73$_{\pm1.88}$	&43.79$_{\pm1.16}$	&47.72$_{\pm1.28}$	&46.38$_{\pm2.37}$	&7.93$_{\pm0.00}$	&[\textbf{59.14}$_{\pm0.00}$]	\\
				Yale	&[\textbf{100.00}$_{\pm0.00}$]	&[\textbf{100.00}$_{\pm0.00}$]	&[\textbf{100.00}$_{\pm0.00}$]	&88.07$_{\pm10.62}$	&99.47$_{\pm1.36}$	&92.93$_{\pm0.00}$	&71.67$_{\pm12.21}$	&85.13$_{\pm0.00}$	&[\textbf{100.00}$_{\pm0.00}$]	&[\textbf{100.00}$_{\pm0.00}$]	\\
				NG-20	&25.47$_{\pm0.00}$	&27.57$_{\pm0.40}$	&\textbf{36.57}$_{\pm5.57}$	&28.24$_{\pm2.19}$	&25.16$_{\pm0.02}$	&26.46$_{\pm0.40}$	&25.29$_{\pm0.27}$	&25.74$_{\pm0.00}$	&28.54$_{\pm0.00}$	&[\textbf{41.74}$_{\pm0.00}$]	\\
				Abalone	&24.31$_{\pm0.25}$	&\textbf{27.75}$_{\pm0.32}$	&19.97$_{\pm0.26}$	&27.71$_{\pm0.36}$	&25.78$_{\pm1.56}$	&19.99$_{\pm0.35}$	&[\textbf{27.77}$_{\pm0.28}$]	&17.00$_{\pm0.00}$	&26.67$_{\pm0.00}$	&26.12$_{\pm0.00}$	\\
				LR	&N/A	&27.24$_{\pm1.06}$	&6.52$_{\pm0.10}$	&27.75$_{\pm1.81}$	&\textbf{29.03}$_{\pm4.40}$	&24.09$_{\pm0.80}$	&28.74$_{\pm1.19}$	&11.01$_{\pm0.05}$	&18.91$_{\pm0.00}$	&[\textbf{30.18}$_{\pm0.00}$]	\\
				YTF-50	&N/A	&N/A	&7.32$_{\pm0.06}$	&59.97$_{\pm2.69}$	&64.34$_{\pm5.90}$	&24.34$_{\pm1.71}$	&\textbf{68.73}$_{\pm1.67}$	&67.74$_{\pm0.00}$	&42.58$_{\pm0.00}$	&[\textbf{76.82}$_{\pm0.00}$]	\\
				YTF-100	&N/A	&N/A	&5.27$_{\pm0.05}$	&45.34$_{\pm1.46}$	&58.97$_{\pm1.45}$	&\textbf{62.34}$_{\pm1.21}$	&54.21$_{\pm1.28}$	&51.41$_{\pm0.00}$	&9.19$_{\pm0.00}$	&[\textbf{73.88}$_{\pm0.00}$]	\\
				\midrule
				Avg.score	&-	&-	&31.90	&43.13	&\textbf{45.19}	&39.36	&43.83	&41.42	&30.87	&[\textbf{54.62}]	\\
				Avg.rank	&7.13	&4.75	&5.88	&5.38	&5.75	&6.25	&\textbf{4.63}	&6.38	&5.88	&[\textbf{1.50}]	\\
				\bottomrule
			\end{tabular}
		}
	\begin{tablenotes}
		\footnotesize
		\item[1] Note that N/A indicates the out-of-memory error.
	\end{tablenotes}
\end{table*}

\subsection{Comparison Results and Analysis}
In this section, we report and analyze the comparison results of OBCut and the nine baseline subspace/spectral clustering methods on the eight benchmark datasets. The clustering scores w.r.t. NMI, ACC, and PUR are given in Tables~\ref{tab:NMI}, \ref{tab:ACC}, and \ref{tab:PUR}, respectively. 

As shown in Table \ref{tab:NMI}, OBCut achieves the best NMI scores on seven out of the eight datasets. Although SSC-OMP outperforms OBCut on the MPEG-7 dataset w.r.t. NMI, yet on all the other seven datasets OBCut yields better or significantly better clustering performance than SSC-OMP. 
In comparison with the other bipartite graph based methods (namely, LSC and U-SPEC), whose anchors are fixed after initialization, our OBCut method can automatically learn a set of anchors during the optimization process and exhibits better NMI scores than LSC and U-SPEC on all the eight datasets. The influence of the joint learning of the anchors and the bipartite graph in OBCut will further be evaluated in Section~\ref{sec:ablation1}.

As shown in Tables \ref{tab:ACC} and \ref{tab:PUR}, in terms of ACC and PUR, OBCut is also able to produce the best or almost the best clustering performance on most of the benchmark datasets. Besides the comparison of the clustering scores on each dataset, we further provide the average ranks and average scores (across all datasets) by different clustering methods at the bottoms of Tables~\ref{tab:NMI}, \ref{tab:ACC}, and \ref{tab:PUR}. In terms of the average score, our OBCut method obtains average scores (across all datasets) of 54.77, 50.91, and 54.62, w.r.t. NMI(\%), ACC(\%), and PUR(\%), respectively, while the second best method only achieves average scores of 43.31, 42.40, and 45.19, respectively. In terms of the average rank, OBCut obtains average ranks of 1.13, 1.38, and 1.50, w.r.t. NMI(\%), ACC(\%), and PUR(\%), respectively, which substantially outperforms the second best method whose average ranks are 4.75, 5.00, and 4.63, respectively. To summarize, the comparison results in Tables~\ref{tab:NMI}, \ref{tab:ACC}, and \ref{tab:PUR} have confirmed the advantageous clustering performance of the proposed OBCut method over the state-of-the-art subspace/spectral clustering methods.

\subsection{Parameter Analysis}
\label{sec:parameter}
In this section, we test the influence of the trade-off parameter $\lambda$ and the number of anchors $M$ in OBCut on four benchmark datasets.

\begin{table}
	\centering 
	\caption{The clustering performance of OBCut with varying values of parameters $\lambda$ and $M$ on the benchmark datasets.}\vskip -0.05 in
	\label{tab:sensitivity}
	\begin{threeparttable}
		\begin{tabular}{m{0.88cm}<{\centering}|m{1.45cm}<{\centering}m{1.45cm}<{\centering}m{1.45cm}<{\centering}m{1.55cm}<{\centering}}
			\toprule
			Dataset   &MPEG-7  &Yale  &Abalone &LR \\
			\midrule
			\multirow{1}{*}{NMI(\%)}
			&\includegraphics[width=1.8cm]{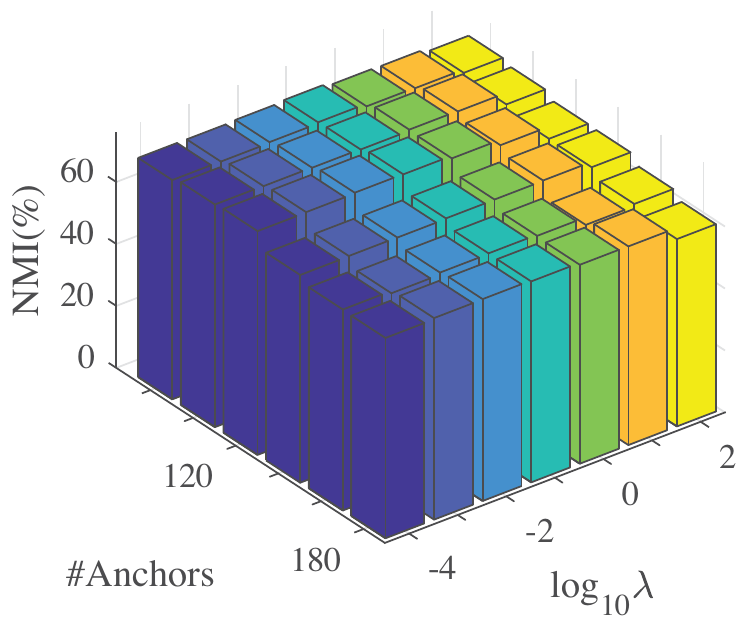}
			&\includegraphics[width=1.8cm]{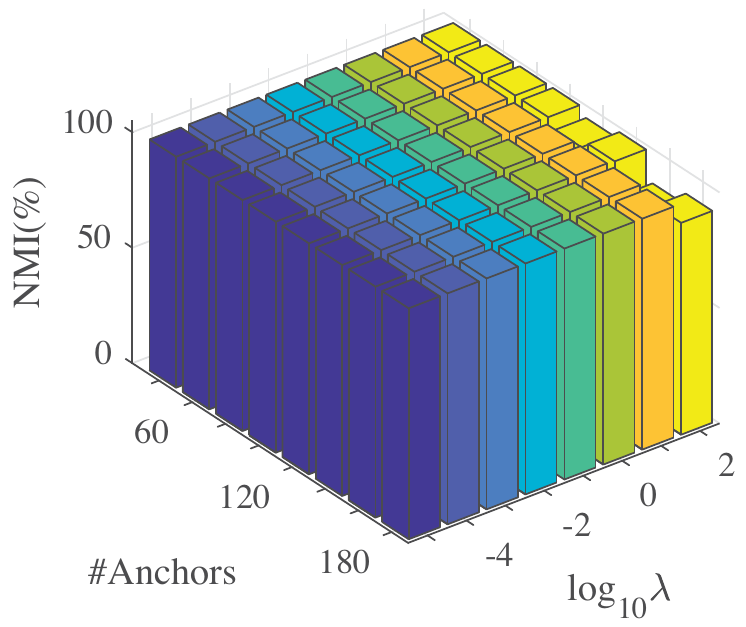}
			&\includegraphics[width=1.8cm]{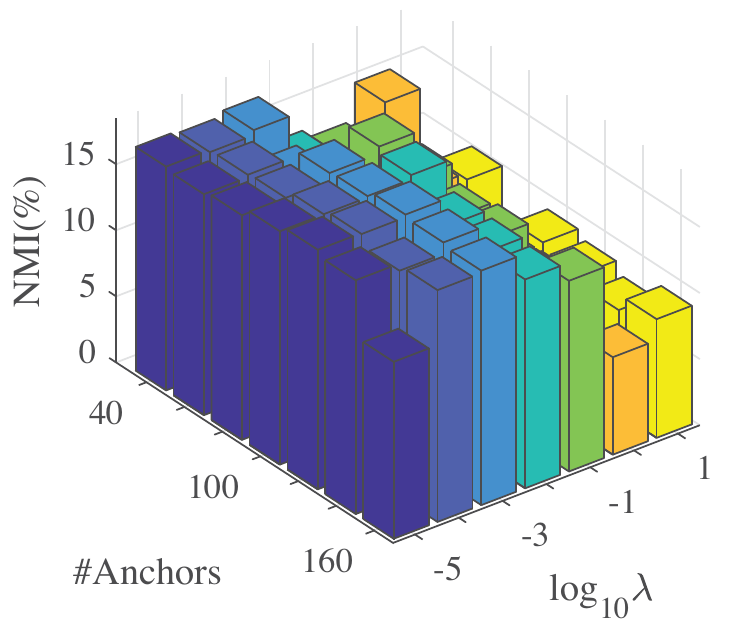}
			&\includegraphics[width=1.8cm]{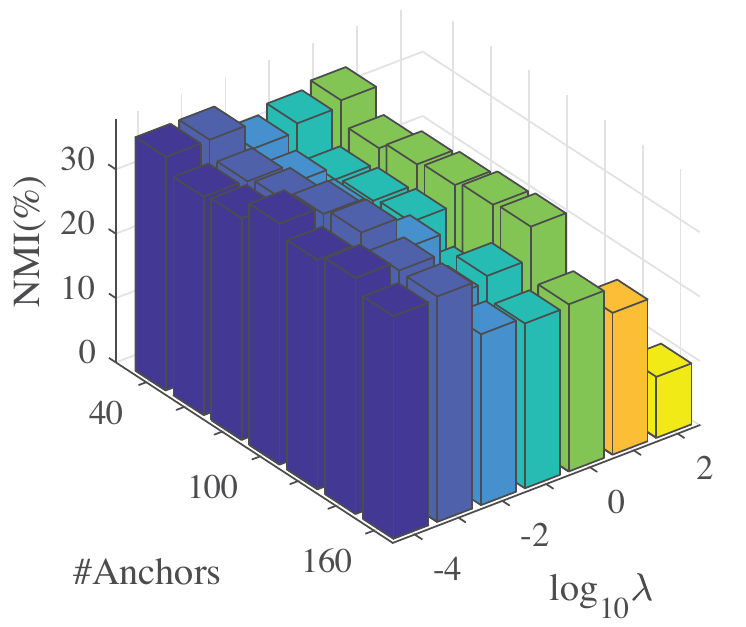}\\
			ACC(\%)
			&\includegraphics[width=1.8cm]{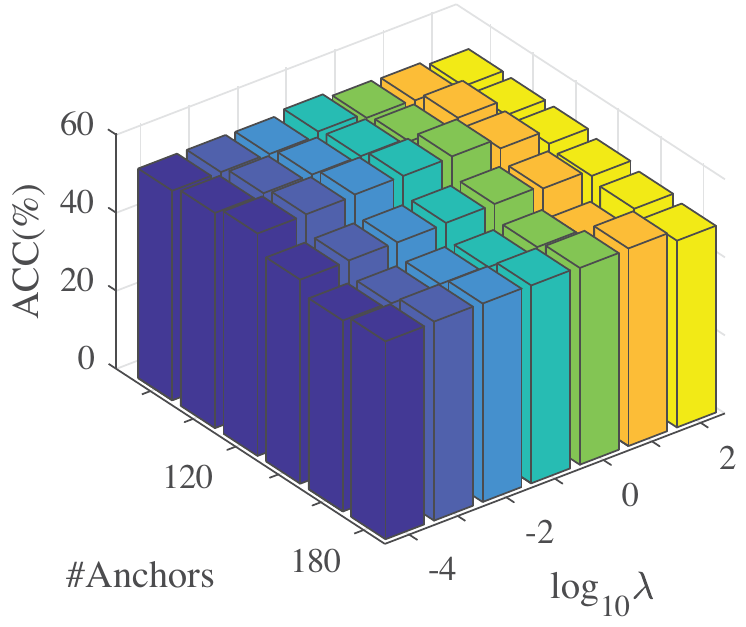}
			&\includegraphics[width=1.8cm]{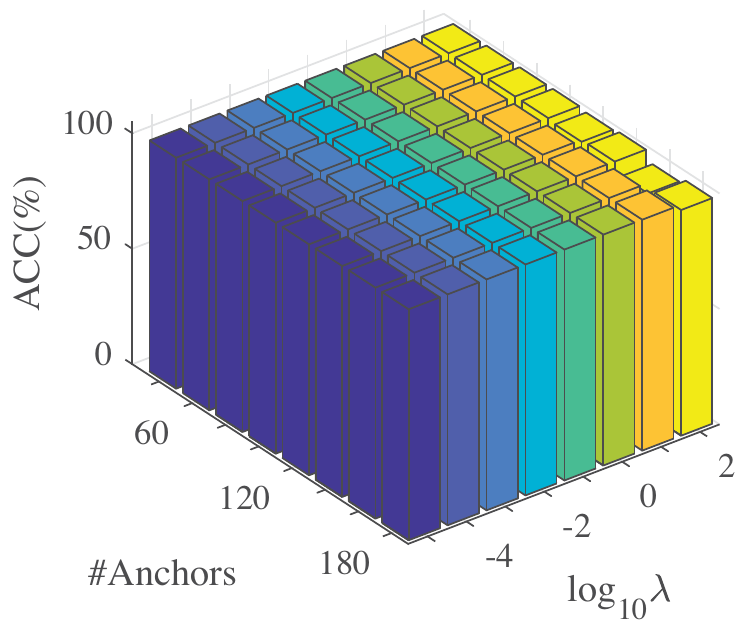}
			&\includegraphics[width=1.8cm]{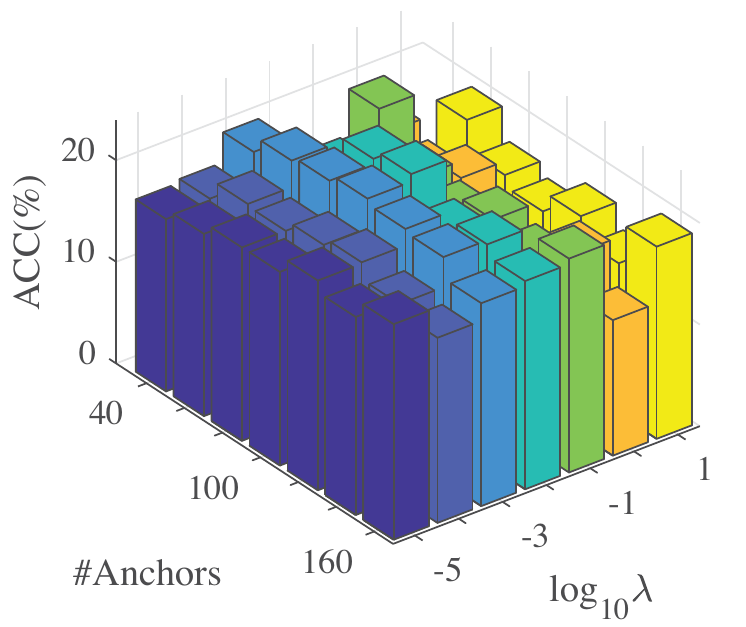}
			&\includegraphics[width=1.8cm]{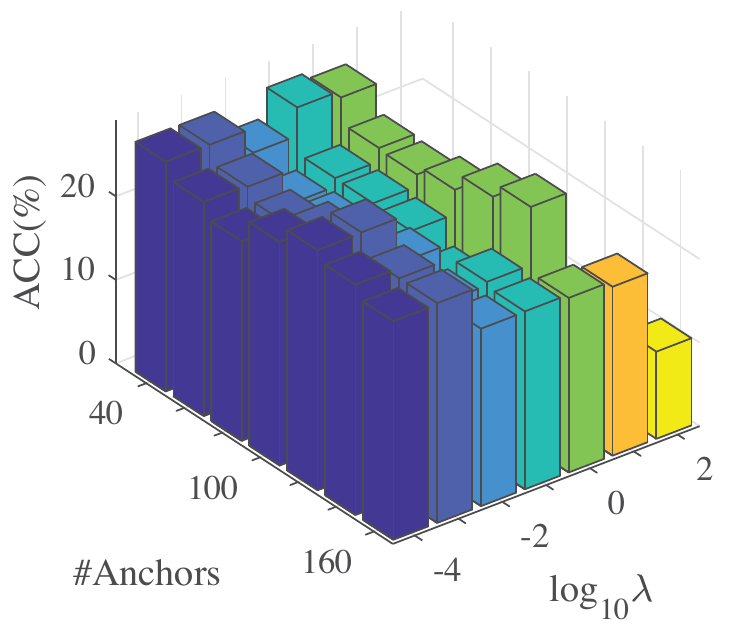}\\
			PUR(\%)
			&\includegraphics[width=1.8cm]{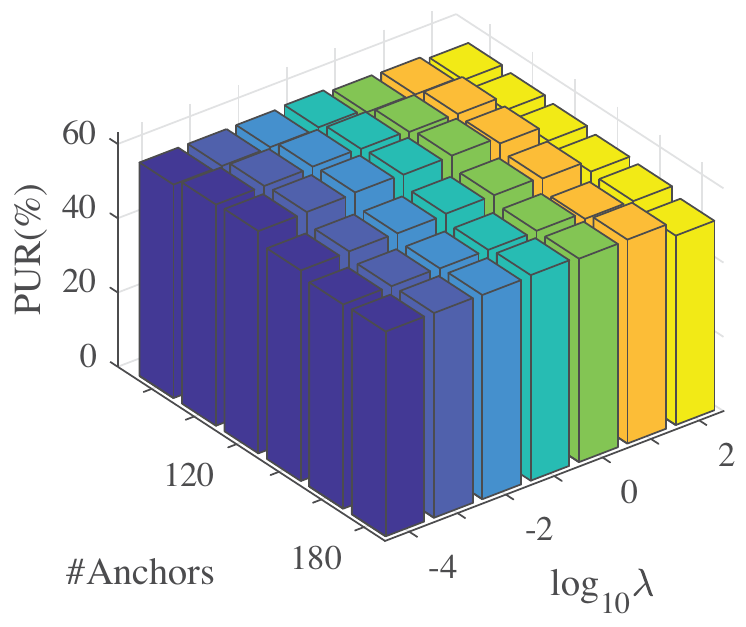}
			&\includegraphics[width=1.8cm]{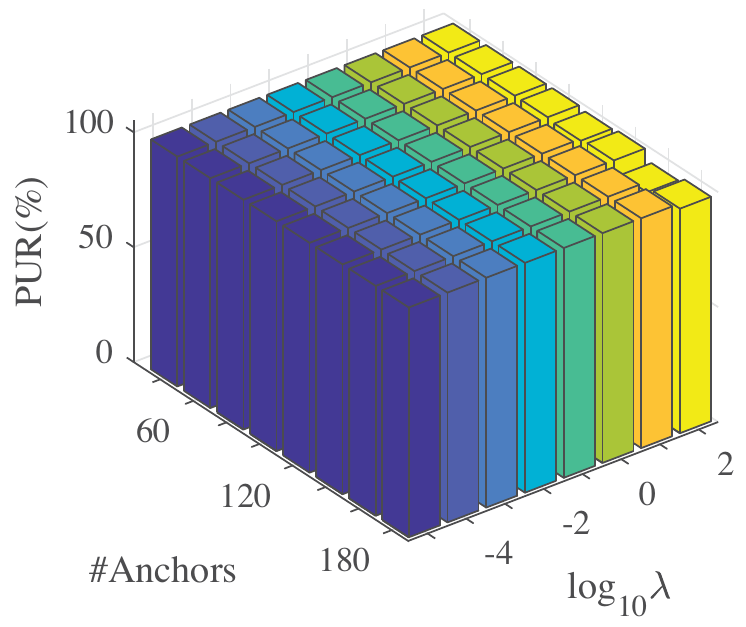}
			&\includegraphics[width=1.8cm]{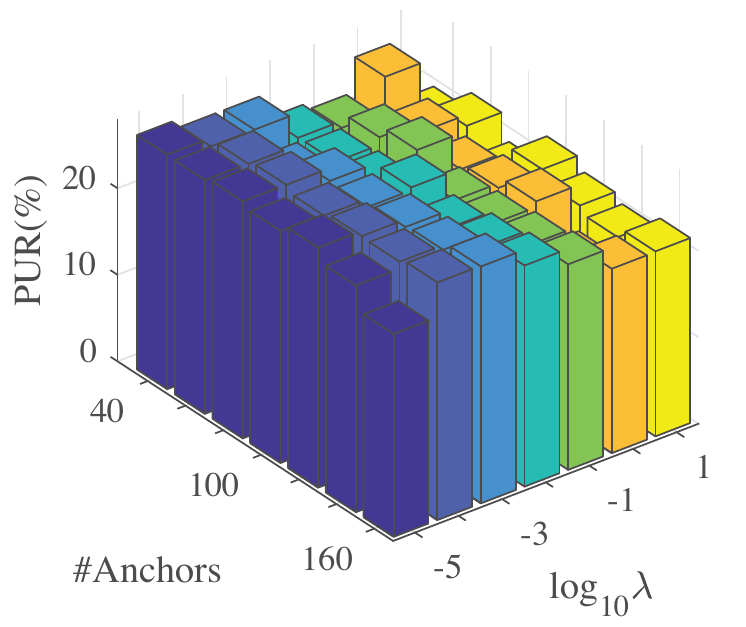}
			&\includegraphics[width=1.8cm]{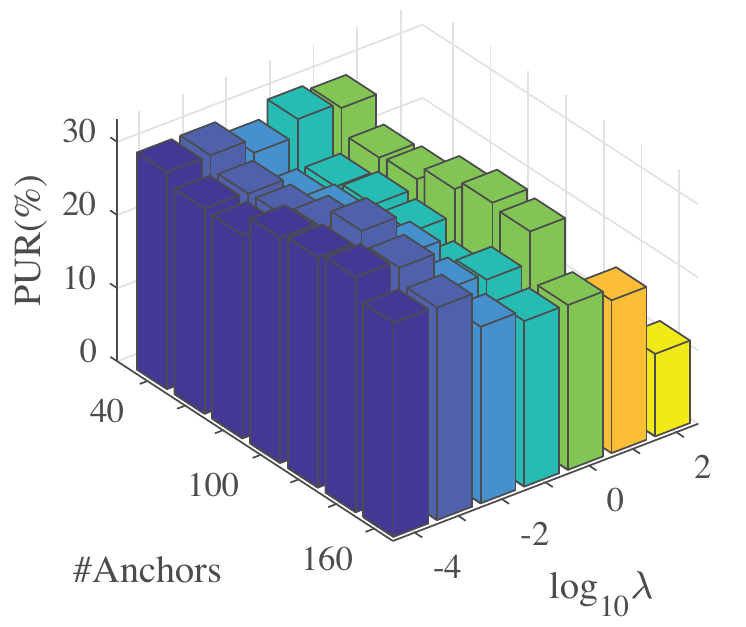}\\
			\bottomrule
		\end{tabular}
	\end{threeparttable}
\end{table}

Specifically, we illustrate the clustering performance of OBCut (w.r.t. NMI, ACC, and PUR) in Table~\ref{tab:sensitivity} with varying values of $\lambda$ and $M$.
In terms of the trade-off parameter $\lambda$, moderate or relatively small values of $\lambda$ are usually beneficial to the clustering performance, which suggests the balance between the adaptive bipartite graph learning term and the bipartite graph partitioning term in OBCut. In terms of the number of anchors $M$, our OBCut method yields quite consistent clustering performance with varying number of anchors. In practice, the tuning of the number of anchors is not a necessary issue. In this work, we use $M=100$ on all the benchmark datasets.

\subsection{Influence of Adaptive Learning of Anchors and Bipartite Graph}
\label{sec:ablation1}
In the proposed OBCut method, the adaptive anchor learning and the bipartite graph learning are simultaneously enforced. 
In this section, we test the influence of the adaptive learning of the anchors and the bipartite graph.

Specifically, three versions (or variants) of our OBCut method are compared. First, the proposed OBCut method with learnable anchors and learnable bipartite graph is denoted as OBCut(LA+LG). Second, the variant using fixed anchor set and learnable bipartite graph is denoted as OBCut(FA+LG). Third, the variant using fixed anchor set and fixed bipartite graph is denoted as OBCut(FA+FG). As shown in Table \ref{table:ablationM1}, when comparing the performances of OBCut(LA+LG) and OBCut(FA+LG), it can be observed that the incorporation of adaptive anchor learning generally leads to more robust clustering performance than using fixed anchors. When comparing the performances of OBCut(LA+LG) and OBCut(FA+FG), it can be observed that the joint learning of the anchor set and the bipartite graph can significantly benefit the clustering performance.

\begin{table}
	\centering \vskip 0.13 in
	\caption{The clustering performance of OBCut with or without the adaptive learning of the anchors and the bipartite graph (\textbf{FA}: Fixed Anchors; \textbf{FG}: Fixed Bipartite Graph; \textbf{LA}: Learnable Anchors; \textbf{LG}: Learnable Bipartite Graph).}\vskip -0.05 in
	\label{table:ablationM1}
	\begin{threeparttable}
		\begin{tabular}{m{0.88cm}<{\centering}|m{1.45cm}<{\centering}m{1.45cm}<{\centering}m{1.45cm}<{\centering}m{1.55cm}<{\centering}}
			\toprule
			Dataset   &MPEG-7  &Yale  &Abalone &LR \\
			\midrule
			\multirow{1}{*}{NMI(\%)}
			&\includegraphics[width=1.78cm]{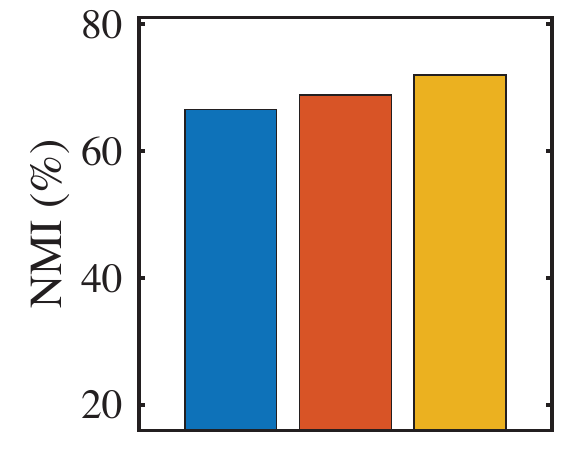}
			&\includegraphics[width=1.78cm]{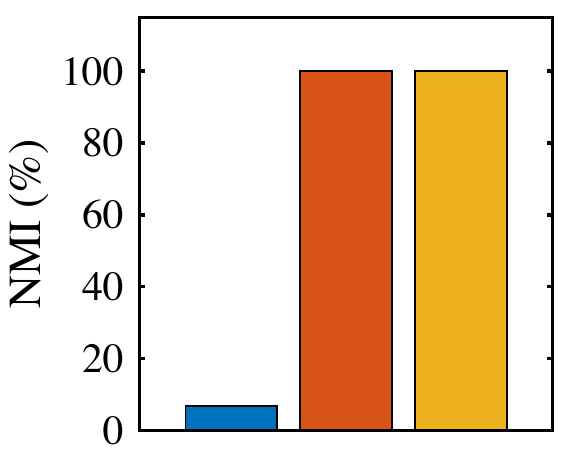}
			&\includegraphics[width=1.78cm]{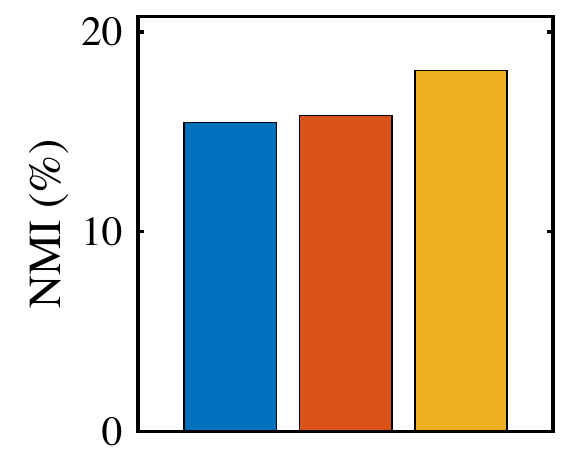}
			&\includegraphics[width=1.78cm]{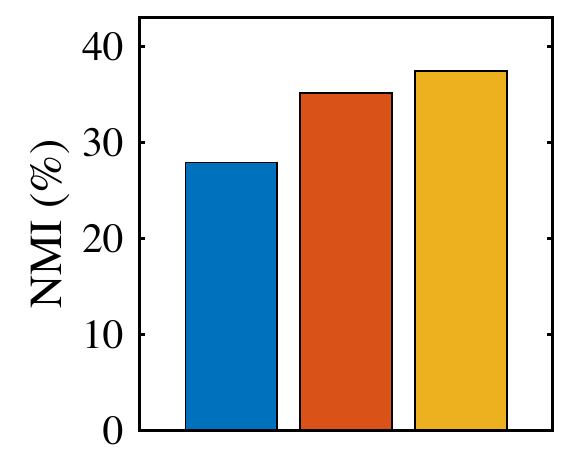}\\
			ACC(\%)
			&\includegraphics[width=1.78cm]{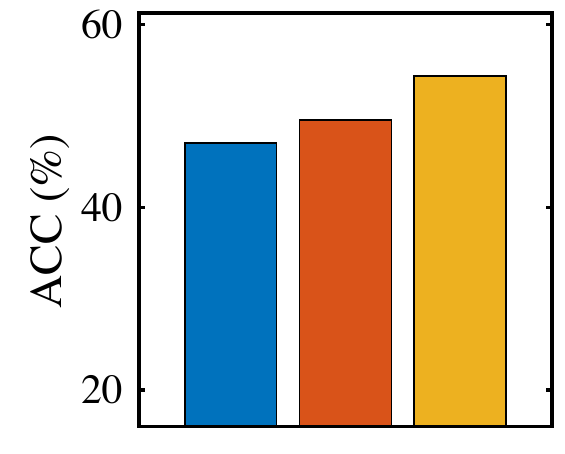}
			&\includegraphics[width=1.78cm]{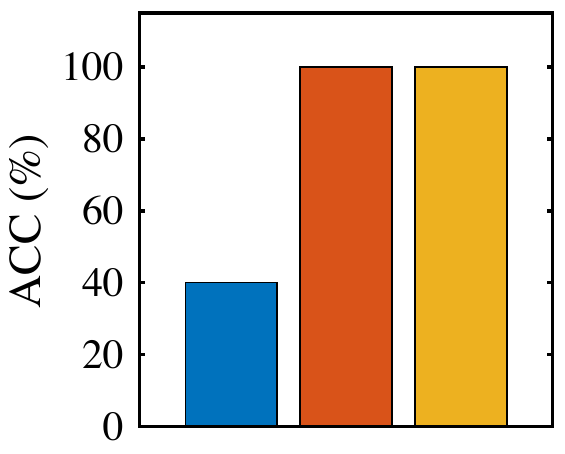}
			&\includegraphics[width=1.78cm]{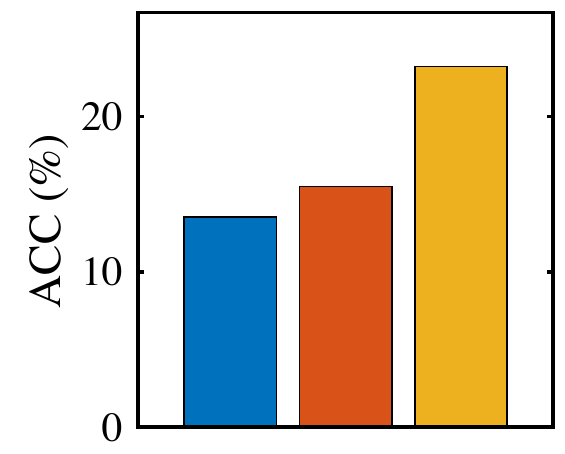}
			&\includegraphics[width=1.78cm]{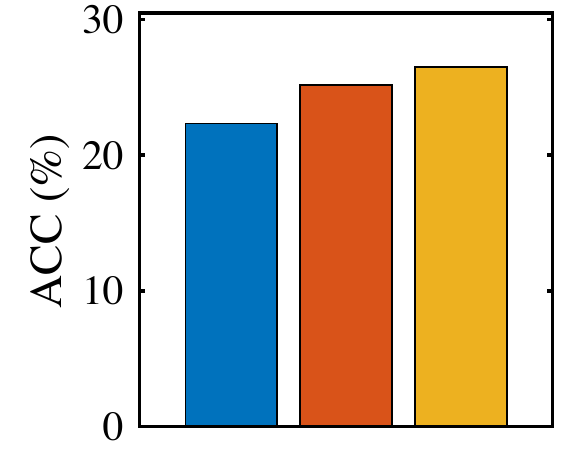}\\
			PUR(\%)
			&\includegraphics[width=1.78cm]{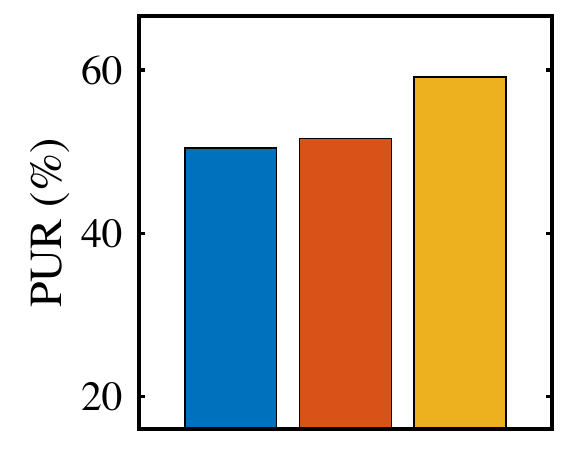}
			&\includegraphics[width=1.78cm]{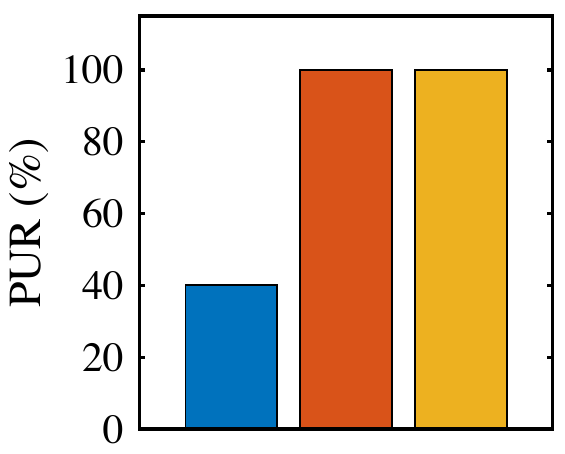}
			&\includegraphics[width=1.78cm]{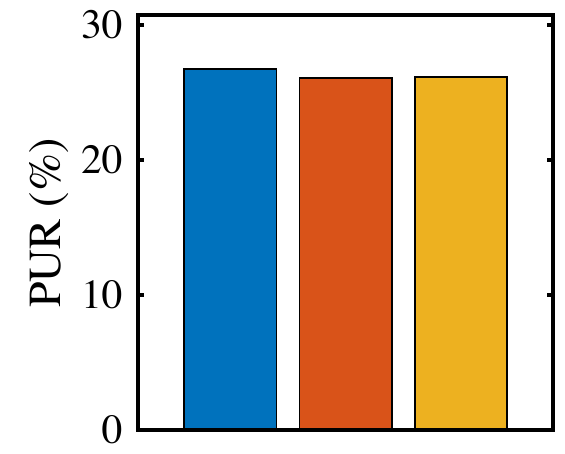}
			&\includegraphics[width=1.78cm]{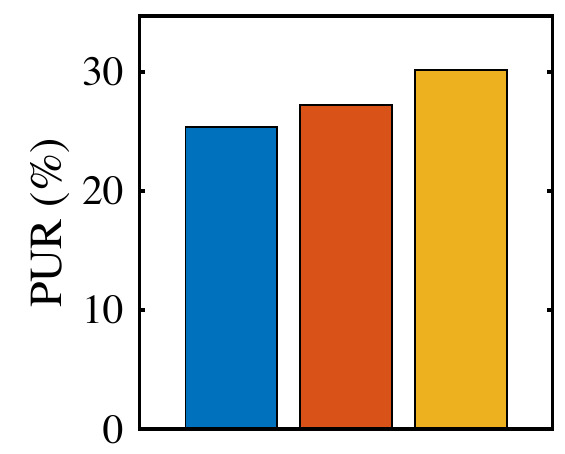}\\
			&\multicolumn{4}{c}{\includegraphics[width=6.3cm]{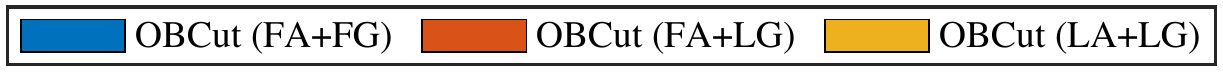}}\\
			\bottomrule
		\end{tabular}\vskip 0.07in
	\end{threeparttable}
\end{table}

\subsection{Influence of the Unified Formulation}
In the objective function \eqref{eq:obj-final} of OBCut, the adaptive anchor and bipartite graph learning is enforced via the first subspace learning term, while the normalized bipartite graph partitioning term is enabled via the second term. In this section, we test the influence of the joint modeling of bipartite graph learning and bipartite graph partitioning in OBCut. Specifically, by treating the bipartite graph learning and the bipartite graph partitioning as two separate steps, we can have the variant called OBCut(two-step formulation). As shown in Table~\ref{table:ablationM2}, especially on the Yale and Abalone datasets, the proposed method with the unified formulation outperforms the two-step variant w.r.t. NMI and ACC by a significant margin. From the experimental results on the four test datasets, it can be observed that the unified formulation of OBCut is able to yield overall more robust clustering performance than the two-step variant.

\begin{table}\vskip 0.19 in
	\centering 
	\caption{The clustering performance of OBCut with unified (one-step) formulation against two-step formulation.}\vskip -0.05 in
	\label{table:ablationM2}
	\begin{threeparttable}
		\begin{tabular}{m{0.88cm}<{\centering}|m{1.45cm}<{\centering}m{1.45cm}<{\centering}m{1.45cm}<{\centering}m{1.55cm}<{\centering}}
			\toprule
			Dataset   &MPEG-7  &Yale  &Abalone &LR \\
			\midrule
			\multirow{1}{*}{NMI(\%)}
			&\includegraphics[width=1.78cm]{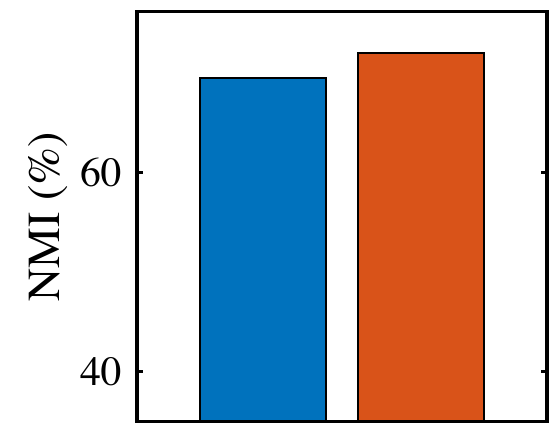}
			&\includegraphics[width=1.78cm]{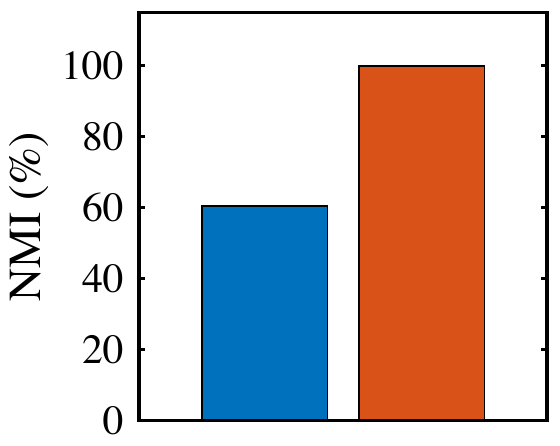}
			&\includegraphics[width=1.78cm]{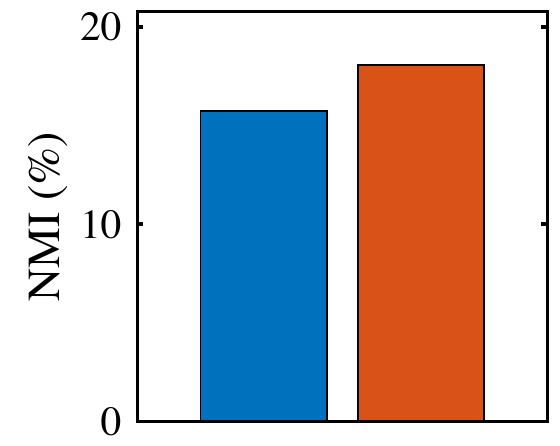}
			&\includegraphics[width=1.78cm]{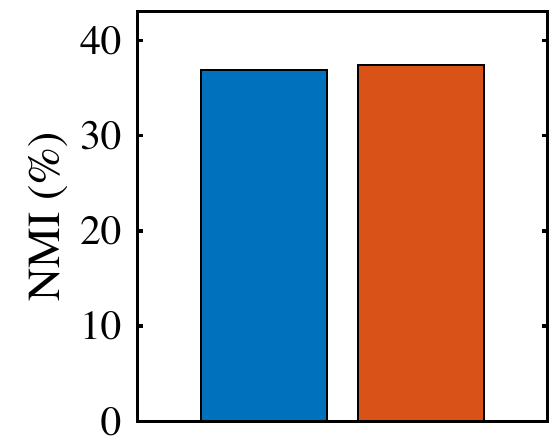}\\
			ACC(\%)
			&\includegraphics[width=1.78cm]{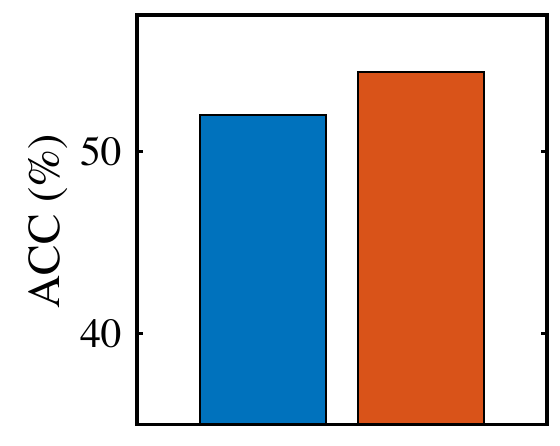}
			&\includegraphics[width=1.78cm]{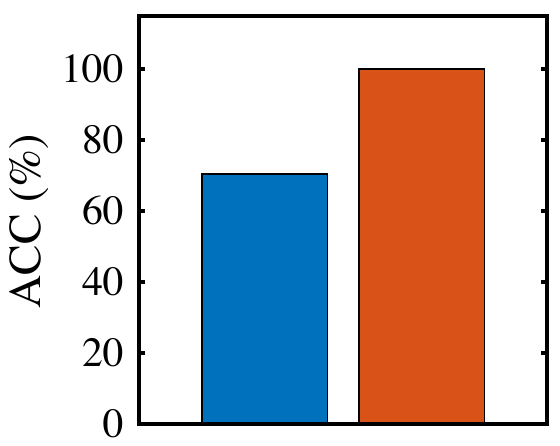}
			&\includegraphics[width=1.78cm]{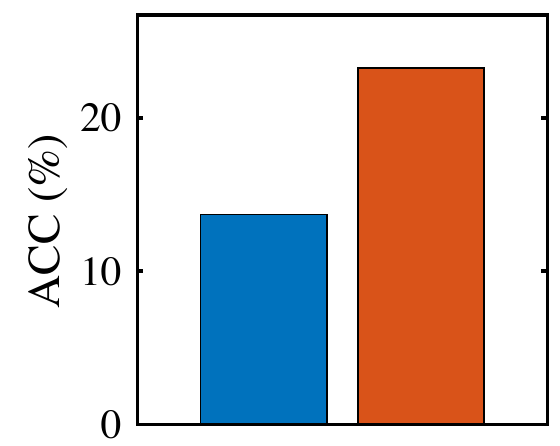}
			&\includegraphics[width=1.78cm]{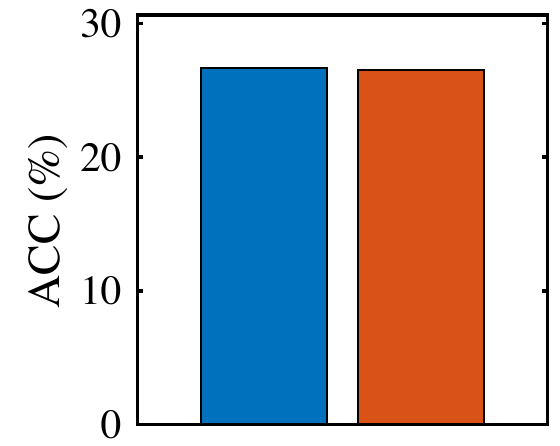}\\
			PUR(\%)
			&\includegraphics[width=1.78cm]{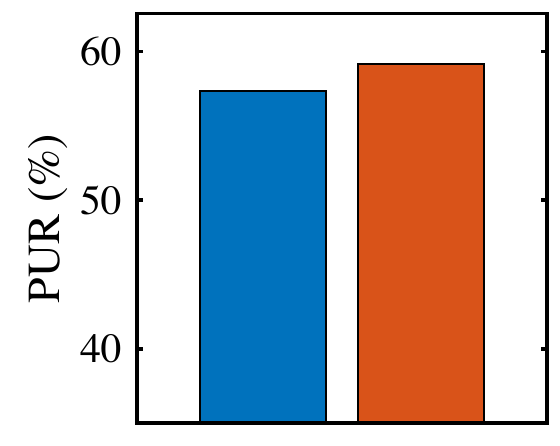}
			&\includegraphics[width=1.78cm]{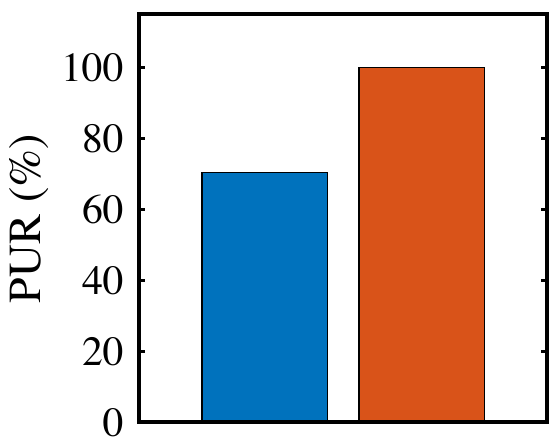}
			&\includegraphics[width=1.78cm]{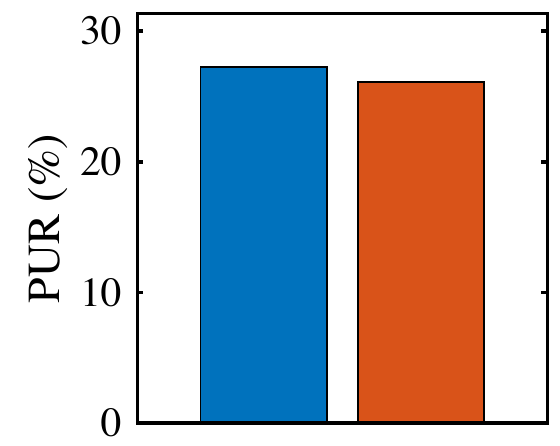}
			&\includegraphics[width=1.78cm]{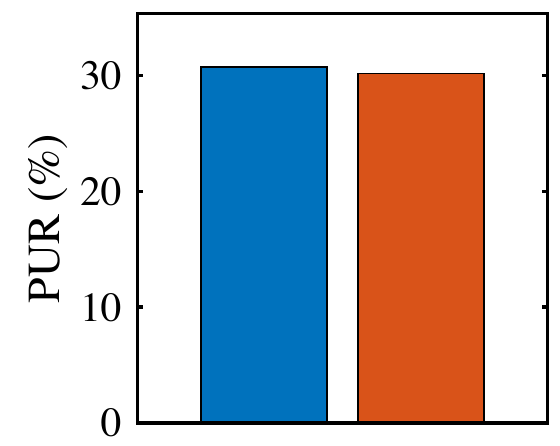}\\
			&\multicolumn{4}{c}{\includegraphics[width=6.4cm]{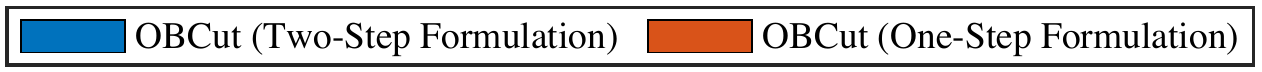}}\\
			\bottomrule
		\end{tabular}
	\end{threeparttable}
\end{table}

\begin{figure}[!t] 
	\centering
	\begin{center}
		{\subfigure[{\scriptsize MPEG-7}]
			{\includegraphics[width=0.442\columnwidth]{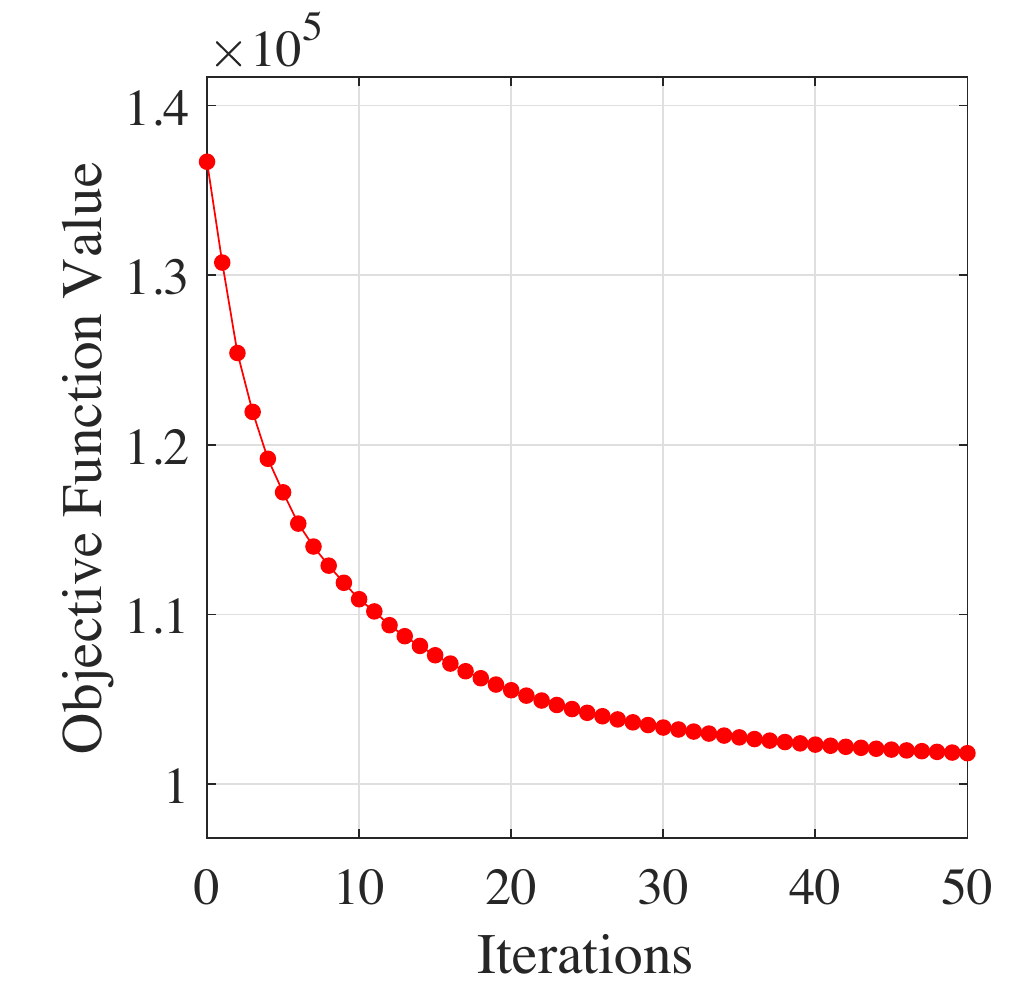}}}
		{\subfigure[{\scriptsize Yale}]
			{\includegraphics[width=0.442\columnwidth]{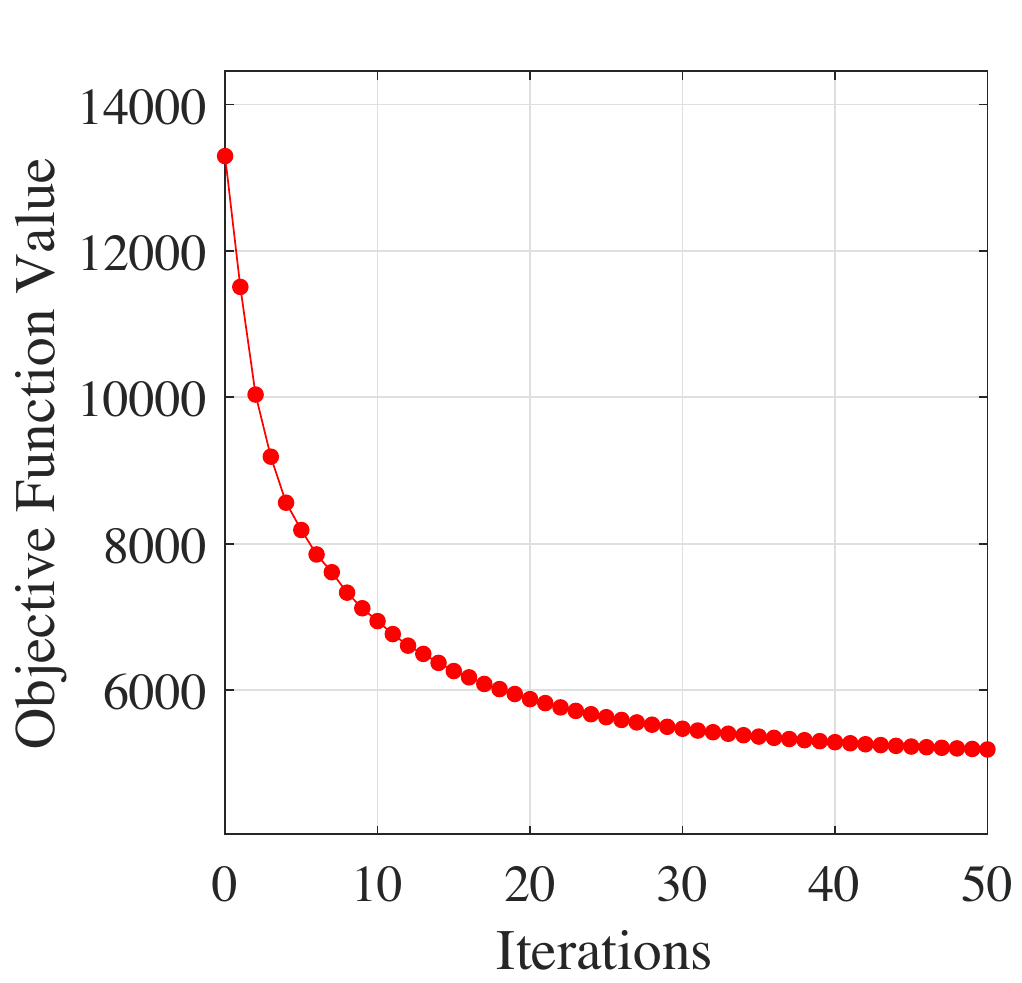}}}\vskip -0.12in
		{\subfigure[{\scriptsize Abalone}]
			{\includegraphics[width=0.442\columnwidth]{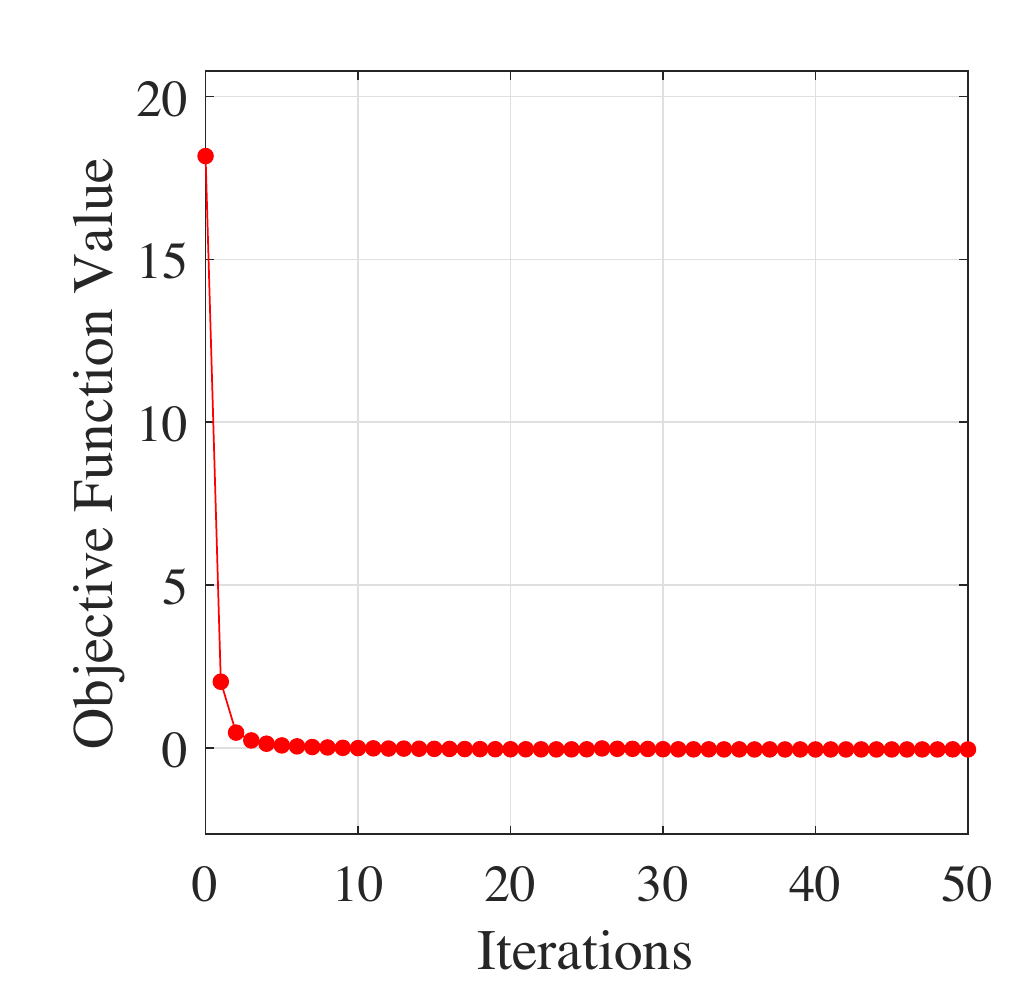}}}
		{\subfigure[{\scriptsize LR}]
			{\includegraphics[width=0.442\columnwidth]{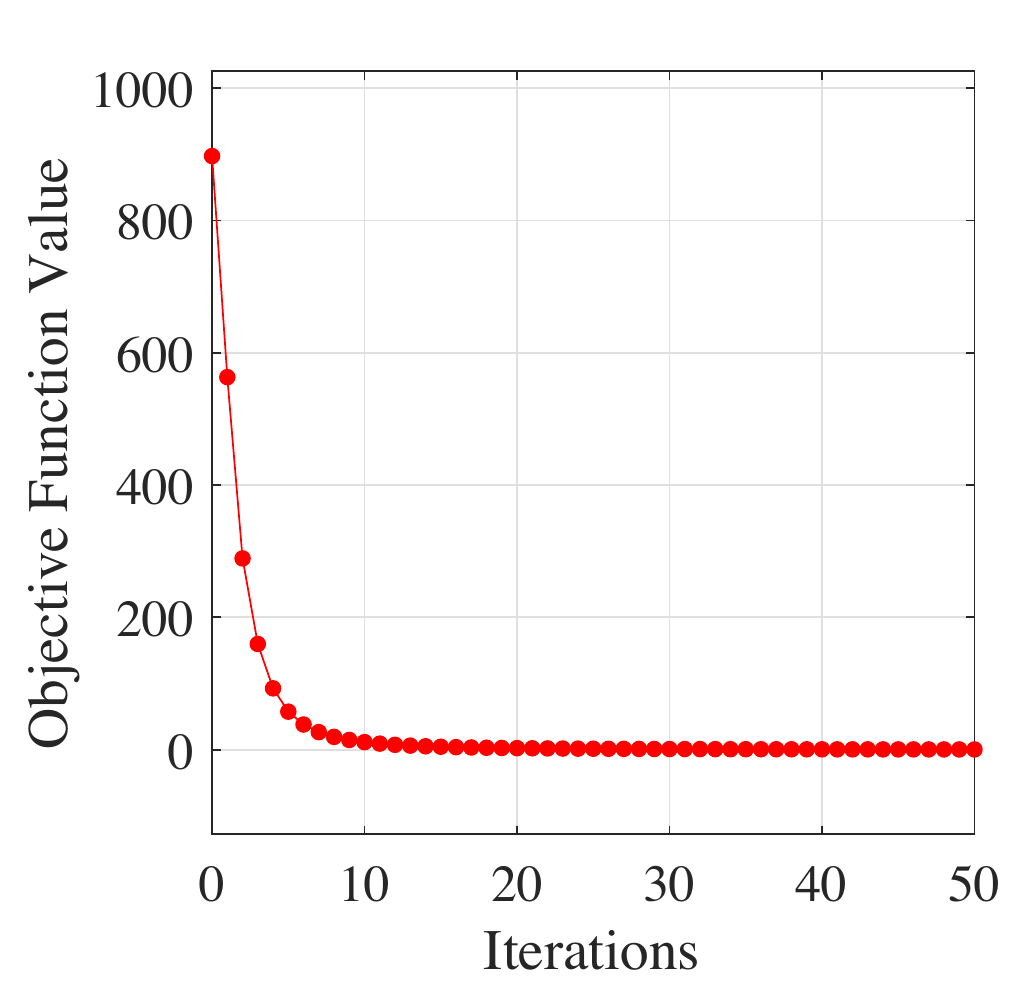}}}\vskip -0.12in
		\caption{Convergence of the objective function value of OBCut with increasing iterations.}
		\label{fig:convergence}
	\end{center}
\end{figure}

\subsection{Convergence Analysis}
\label{sec:empirical_convergence}
In this section, we conduct the convergence analysis on the four benchmark datasets. Figure \ref{fig:convergence} shows the objective function values of OBCut varying with different iterations. As can be seen in this figure, the objective function values monotonically decrease and rapidly converge as the number of iterations grows, which shows the good convergence property of our proposed OBCut method.

\subsection{Execution Time}
In this section, we evaluate the time efficiency of different clustering methods. As the YTF-100 dataset consists of 195,537 data samples, we test the execution times of different methods with different subsets of YTF-100, whose sizes go from 10,000 to the full size of 195,537. As shown in Fig.~\ref{fig:time}, SSC and ESCG are not computationally feasible for the full dataset of YTF-100 due to the out-of-memory error. For the other methods, OBCut is faster than DCDP-ASC, RKSC and EulerSC, and slower than U-SPEC, FastESC, and LSC, probably due to the fact that U-SPEC and LSC utilize the predefined bipartite graph but lack the bipartite graph learning process. 

To conclude the experimental analysis, the proposed OBCut method is able to achieve significantly better clustering performance than the baseline methods (as shown in Tables~\ref{tab:NMI}, \ref{tab:ACC}, and \ref{tab:PUR}) while maintaining competitive efficiency for very large-scale datasets (as shown in Fig.~\ref{fig:time}).

\begin{figure}[!t] 
	\centering
	\begin{center}
		{\subfigure[{\scriptsize Original Scale}]
			{\label{fig:time-a}\includegraphics[width=0.5128\columnwidth]{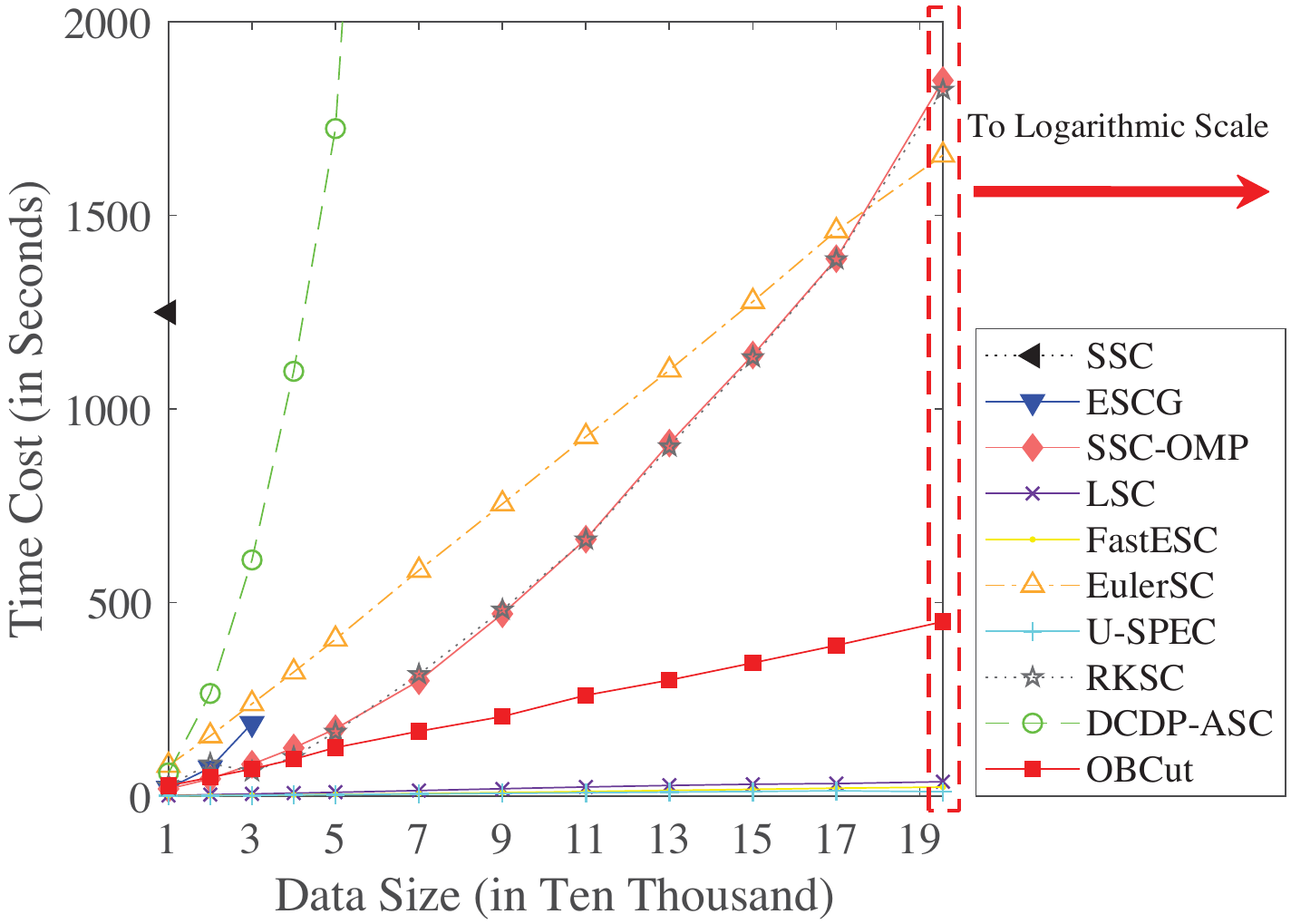}}}
		{\subfigure[{\scriptsize Logarithmic Scale}]
			{\label{fig:time-b}\includegraphics[width=0.4772\columnwidth]{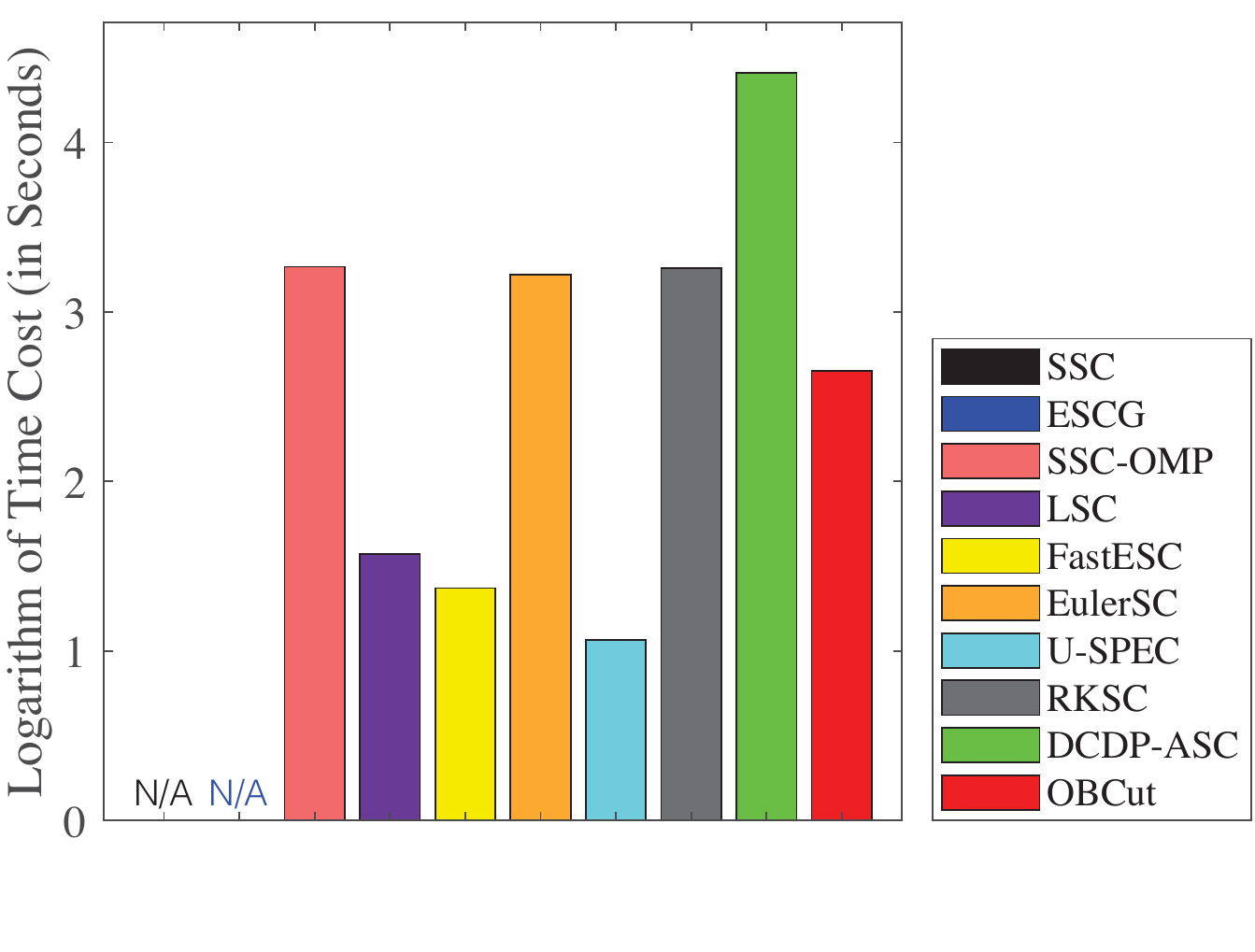}}}\vskip -0.09in
		\caption{Time costs of different clustering methods on the YTF-100 dataset with the data size varying from $10,000$ to $195,537$.}\vskip -0.09in
		\label{fig:time}
	\end{center}
\end{figure}

\section{Conclusion}\label{sec:conclusion}
In this paper, we propose a new scalable subspace clustering approach based on one-step bipartite graph cut (OBCut). In particular, we first characterize a one-step normalized bipartite graph cut criterion, and theoretically prove its equivalence to a trace maximization problem. Based on the new bipartite graph cut criterion, by simultaneously modeling adaptive anchor learning, bipartite graph learning (via subspace learning), and normalized bipartite graph partitioning in a joint learning framework, we can directly achieve a discrete clustering solution in a one-step formulation. An alternating optimization algorithm is designed to solve this joint learning problem, whose time complexity is linear to the sample size. Extensive experiments on eight real-world datasets have demonstrated the superiority of our OBCut approach over the state-of-the-art subspace/spectral clustering approaches.

\section*{Acknowledgments}

This project was supported by the NSFC (61976097, 62276277  \& U22A2095), and the Natural Science Foundation of Guangdong Province (2021A1515012203).

\bibliographystyle{IEEEtran}
\bibliography{BiCut-ref}

\begin{IEEEbiography}[{\includegraphics[width=1in,height=1.25in,clip,keepaspectratio]{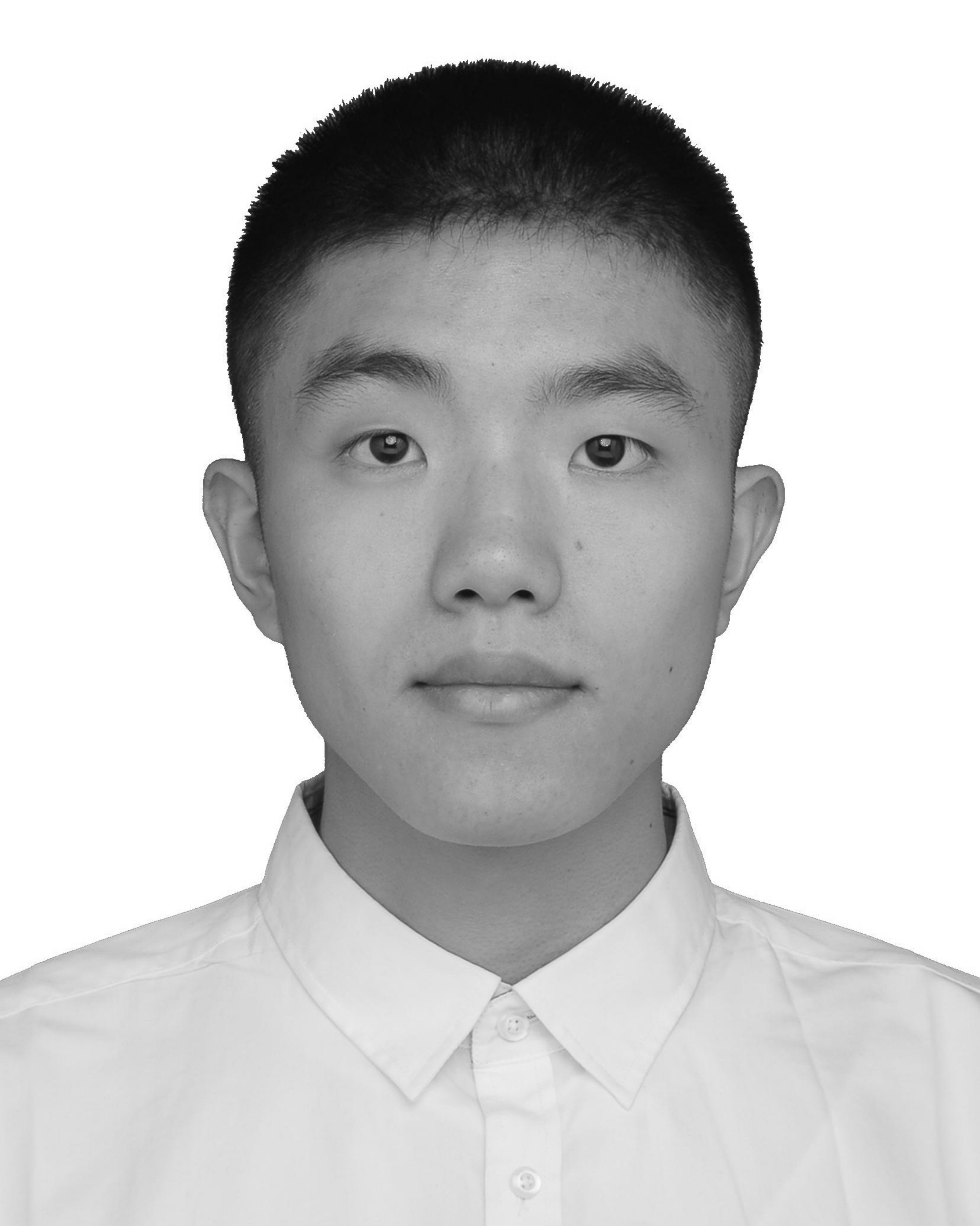}}]{Si-Guo Fang}
	received the B.S. degree in information and computing sciences from the China Jiliang University, Hangzhou, China, in 2021. He is currently pursuing the master degree in computer science with the College of Mathematics and Informatics, South China Agricultural University, Guangzhou, China. His research interests include data mining and machine learning.
\end{IEEEbiography}

\begin{IEEEbiography}[{\includegraphics[width=1in,height=1.25in,clip,keepaspectratio]{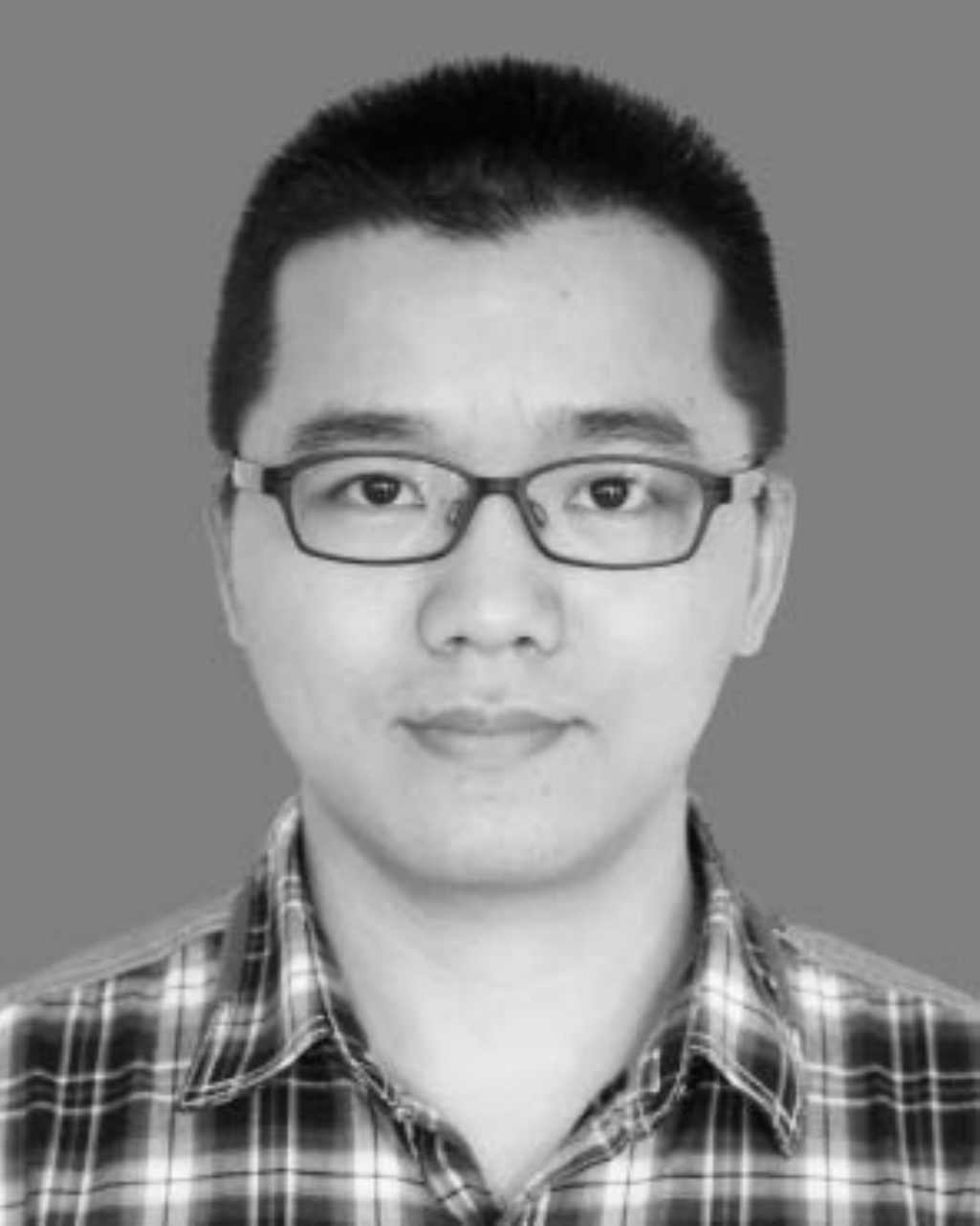}}]{Dong Huang}
	received the B.S. degree in computer science in 2009 from South China University of Technology, Guangzhou, China. He received the M.Sc. degree in computer science in 2011 and the Ph.D. degree in computer science in 2015, both from Sun Yat-sen University, Guangzhou, China. He joined South China Agricultural University in 2015, where he is currently an Associate Professor with the College of Mathematics and Informatics. From July 2017 to July 2018, he was a visiting fellow with the School of Computer Science and Engineering, Nanyang Technological University, Singapore. His research interests include data mining and machine learning. He has published more than 70 papers in international journals and conferences, such as IEEE TKDE, IEEE TNNLS, IEEE TCYB, IEEE TSMC-S, ACM TKDD, SIGKDD, AAAI, and ICDM. He was the recipient of the 2020 ACM Guangzhou Rising Star Award.
\end{IEEEbiography}

\begin{IEEEbiography}[{\includegraphics[width=1in,height=1.25in,clip,keepaspectratio]{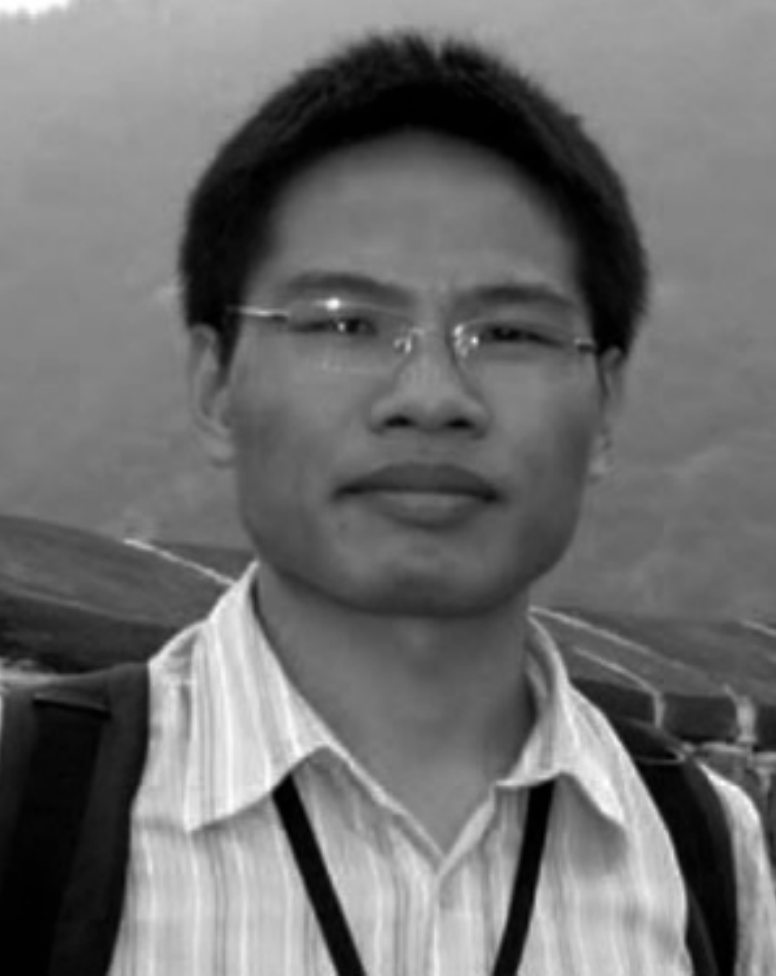}}]{Chang-Dong Wang}
	received the B.S. degree in applied mathematics in 2008, the M.Sc. degree in computer science in 2010, and the Ph.D. degree in computer science in 2013, all from Sun Yat-sen University, Guangzhou, China. He was a visiting student at the University of Illinois at Chicago from January 2012 to November 2012. He is currently an Associate Professor with the School of Data and Computer Science, Sun Yat-sen University, Guangzhou, China. His current research interests include machine learning and data mining. He has published more than 100 scientific papers in international journals and conferences such as IEEE TPAMI, IEEE TKDE, IEEE TNNLS, IEEE TSMC-C, ACM TKDD, Pattern Recognition, SIGKDD, AAAI, ICDM and SDM. His ICDM 2010 paper won the Honorable Mention for Best Research Paper Award. He was awarded 2015 Chinese Association for Artificial Intelligence (CAAI) Outstanding Dissertation.
\end{IEEEbiography}

\begin{IEEEbiography}[{\includegraphics[width=1in,height=1.25in,clip,keepaspectratio]{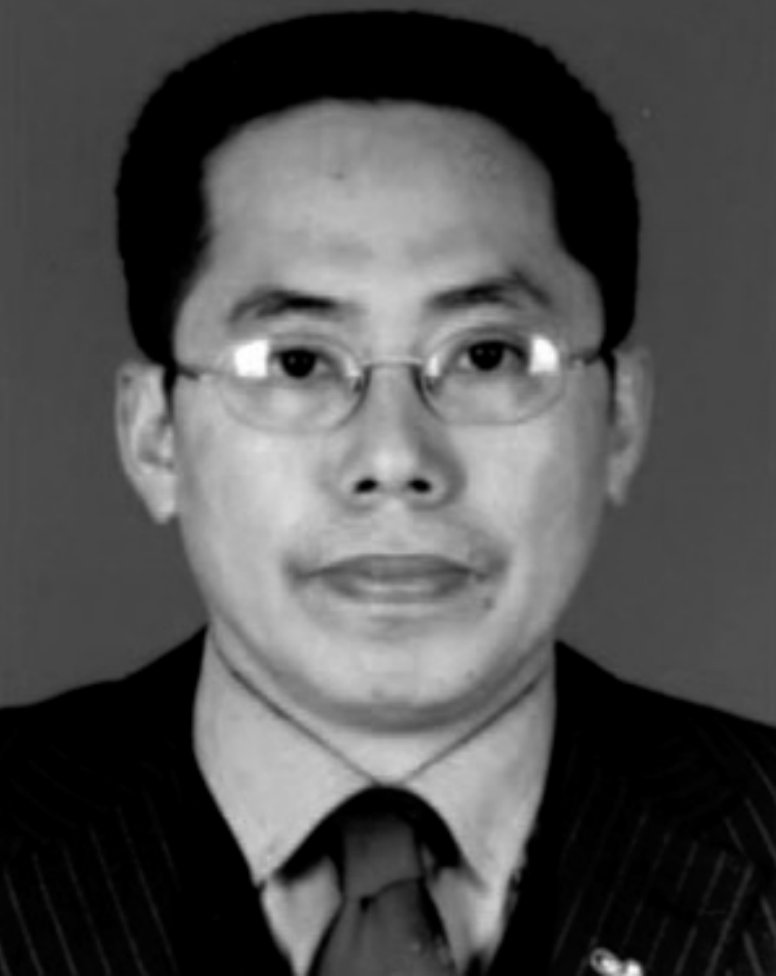}}]{Jian-Huang Lai}
received the M.Sc. degree in applied mathematics in 1989 and the Ph.D. degree in mathematics in 1999 from Sun Yat-sen University, China. He joined Sun Yat-sen University in 1989 as an Assistant Professor, where he is currently a Professor with the School of Data and Computer Science. His current research interests include the areas of digital image processing, pattern recognition, multimedia communication, wavelet and its applications. He has published more than 200 scientific papers in the international journals and conferences on image processing and pattern recognition, such as IEEE TPAMI, IEEE TKDE, IEEE TNN, IEEE TIP, IEEE TSMC-B, Pattern Recognition, ICCV, CVPR, IJCAI, ICDM and SDM. Prof. Lai serves as a Standing Member of the Image and Graphics Association of China, and also serves as a Standing Director of the Image and Graphics Association of Guangdong.
\end{IEEEbiography}

\end{document}